\documentclass[sigconf]{acmart}
\usepackage{graphicx}
\usepackage{colortbl}
\usepackage{xcolor}
\usepackage[T1]{fontenc}
\usepackage{algorithmic}
\usepackage{algorithm}

\renewcommand{\algorithmiccomment}[1]{/*~#1~*/}

\usepackage{todonotes}
\usepackage{booktabs}
\usepackage{bm}
\usepackage{amsmath}
\usepackage{adjustbox}
\usepackage{natbib}
\usepackage{amsfonts}
\usepackage[nameinlink,capitalize]{cleveref}
\usepackage{float}
\usepackage{algorithm}
\usepackage{subcaption}
\usepackage{caption}
\usepackage{dsfont}
\usepackage{multirow}
\usepackage{array}
\usepackage{enumitem}
\usepackage{nccmath}
\usepackage{cleveref}
\usepackage{hyperref}
\usepackage{xspace}
\usepackage{cleveref}
\usepackage{arydshln}
\usepackage{url}
\usepackage{graphicx}

\newcommand{\algname}{MCD-DD}

\theoremstyle{definition}

\newtheorem{theorem}{Theorem}

\newcolumntype{L}[1]{>{\raggedright\let\newline\\\arraybackslash\hspace{0pt}}m{#1}}
\newcolumntype{C}[1]{>{\centering\arraybackslash}p{#1}}

\copyrightyear{2024}
\acmYear{2024}
\setcopyright{rightsretained}
\acmConference[KDD '24]{Proceedings of the 30th ACM SIGKDD Conference on Knowledge Discovery and Data Mining}{August 25--29, 2024}{Barcelona, Spain}
\acmBooktitle{Proceedings of the 30th ACM SIGKDD Conference on Knowledge Discovery and Data Mining (KDD '24), August 25--29, 2024, Barcelona, Spain}
\acmDOI{10.1145/3637528.3672016}
\acmISBN{979-8-4007-0490-1/24/08}

\begin{document}

\title{Online Drift Detection with Maximum Concept Discrepancy}


\author{Ke Wan$^{*}$}
\affiliation{%
  \institution{University of Illinois at Urbana-Champaign}
  \city{Illinois}
  \country{USA}
}
\email{kewan2@illinois.edu}

\author{Yi Liang$^{*}$}
\affiliation{%
  \institution{Fudan University}
  \city{Shanghai}
  \country{China}
}
\email{yliang23@m.fudan.edu.cn}

\author{Susik Yoon\textsuperscript{\S}}
\affiliation{%
  \institution{Korea University}
  \city{Seoul}
  \country{Korea}
}
\email{susik@korea.ac.kr}

\renewcommand{\shortauthors}{Ke Wan, Yi Liang, and Susik Yoon}


\begin{abstract}
Continuous learning from an immense volume of data streams becomes exceptionally critical in the internet era.
However, data streams often do not conform to the same distribution over time, leading to a phenomenon called concept drift.
Since a fixed static model is unreliable for inferring concept-drifted data streams, establishing an adaptive mechanism for detecting concept drift is crucial.
Current methods for concept drift detection primarily assume that the labels or error rates of downstream models are given and/or underlying statistical properties exist in data streams. These approaches, however, struggle to address high-dimensional data streams with intricate irregular distribution shifts, which are more prevalent in real-world scenarios.
In this paper, we propose \algname{}, a novel concept drift detection method based on maximum concept discrepancy, inspired by the maximum mean discrepancy. Our method can adaptively identify varying forms of concept drift by contrastive learning of concept embeddings without relying on labels or statistical properties. With thorough experiments under synthetic and real-world scenarios, we demonstrate that the proposed method outperforms existing baselines in identifying concept drifts and enables qualitative analysis with high explainability.

\end{abstract}

\ccsdesc[500]{Information systems~Data stream mining}

\keywords{Concept Drift Detection; Maximum Concept Discrepancy}

\maketitle

\def\thefootnote{*}\footnotetext{Equal contributions. $^{\S}$ Corresponding author.}\def\thefootnote{\arabic{footnote}}

\section{Introduction}
\subsection{Background and Motivation}
Continuously learning from evolving data streams is crucial for numerous online services to derive real-time insights\,\cite{bifet2023machine, ustory, nets}.
However, in many real-world scenarios, data streams from different times may exhibit distinct characteristics\,\cite{pdsum, trirat2023mg, kim2022covid}. 
For instance, a previously stable weather pattern might incrementally change due to global warming with unprecedented high temperatures, leading to unpredictable fluctuations in temperature, wind speed, and humidity.
This phenomenon is called \emph{concept drift} in data streams\,\cite{lu2018learning}, indicating that data at different times follows distinct probability distributions.
Developing methods for continuously detecting whether a data stream has undergone concept drift is imperative, since it is impractical to employ consistent modeling to the concept-drifted data streams (e.g., a weather prediction model needs to be updated after unprecedented temperature changes are observed). 

Current methods for detecting concept drift online fall into two categories: error rate-based or data distribution-based\,\cite{bayram2022concept, song2007statistical, gama2004learning}. A common tactic involves constructing a hypothesis test statistic to determine whether error rates of downstream models or data samples from different periods adhere to the same probability distribution under a certain significance level. While this approach is favored for its interpretability and strong statistical foundation, distinguishing between natural fluctuations and actual drifts poses challenges, especially in the context of complex, evolving data streams. The sparsity, noise, and high dimensionality commonly observed in real-world data streams can make statistical approaches ineffective. Additionally, obtaining error rates of downstream models is not always feasible, as true labels may not be readily available.

Meanwhile, in machine learning, kernel methods are commonly used to map data into high-dimensional spaces\,\cite{elisseeff2001kernel}, improving its representation for downstream tasks such as classification and clustering with more distinct separations in the projected space. Likewise, kernel methods can be used to transform a set of sampled data into a space where the existing concepts can be effectively represented, facilitating the detection of potential concept drifts. However, traditional kernels like the Gaussian kernel\,\cite{gretton2009fast} are limited in their ability to detect concept drifts. They are designed with a deterministic mapping function, making them ill-suited for the ever-changing distributions of data streams.  While deep kernels offer more flexibility\,\cite{wilson2016deep, liu2020learning}, adapting them to address concurrently evolving concepts, especially in unsupervised settings, remains a significant challenge. The computational costs associated with updating these kernels repeatedly can be prohibitive.


\subsection{Main Idea and Challenges}
Detecting concept drifts from data streams presents numerous challenges, mainly centered around the representation of ever-changing data distributions (i.e., concept representation) and the measurement of their differences (i.e., drift quantification). It also necessitates the continuous monitoring of the dynamic shifts in data distributions as they evolve, which is crucial for accurately identifying arbitrary drifts (i.e., online updates). Furthermore, in real-world scenarios, there is often a lack of ground truth labels for concept drifts as well as downstream tasks, making an unsupervised approach (i.e., data distribution-based) preferable to a supervised approach (i.e., error rate-based) in practice. To address these objectives, we propose a novel method for continuously identifying concept drifts in an 
\emph{unsupervised} and \emph{online} manner, that can effectively handle arbitrary data distributions with high interpretability.

The main idea of this work is to employ a new measure \emph{Maximum Concept Discrepancy} for concept drift detection, inspired by the maximum mean discrepancy\,\cite{smola2006maximum} with a kernel function. Through a deep neural network, we encode a set of sample data points in a short time period into a compact representation that captures the concept observed during the period. We leverage contrastive learning accompanied by time-aware sampling strategies to learn the embedding space of concepts. This entails the generation of positive sample pairs drawn from temporally proximate distributions and that of negative sample pairs from temporally distant distributions, while also introducing controlled perturbations. The embedding space is continuously updated to bring positive samples closer together and push negative samples further apart. Concept drifts are then identified by evaluating the discrepancy between the representations of concepts in consecutive time periods. In addition, the maximum concept discrepancy between two concepts can be bounded with a statistical significance. It can function as a theoretical threshold for detecting concept drifts, providing high interpretability and practicality for our method. In consequence, our method is capable of continuously identifying various types of concept drift from data streams without any supervision. It effectively addresses the aforementioned challenges for concept drift detection while keeping the advantages of both online statistical approaches and offline deep kernel-based approaches.

\subsection{Summary}
As a concrete implementation of our main idea, we propose an algorithm \emph{\algname{} (Maximum Concept Discrepancy-based Drift Detector)}, aiming at unsupervised online concept drift detection from data streams.
The main contributions of this work can be summarized as follows:
\begin{itemize}[leftmargin=10pt, noitemsep]
\item 
To the best of our knowledge, this is the first work to propose a dynamically updated measure, \emph{maximum concept discrepancy}, for unsupervised online concept drift detection.
\item 
We propose a novel method \algname{} equipped with the sample set encoder and drift detector, optimized by contrastive objective with time-aware sampling strategies. For reproducibility, the source code of \algname{} is publicly available\footnote{\url{https://github.com/LiangYiAnita/mcd-dd}}.
\item 
Theoretical analysis of learning the maximum concept discrepancy provides its statistical interpretation and complexity.
\item 
Comprehensive experiments are conducted on 11 data sets with varying complexities of drifts.
\algname{} achieves state-of-the-art results in three performance metrics and demonstrates better interpretability in qualitative analysis, compared with baselines. 
\vspace{-0.1cm}
\end{itemize}

\section{Related Work}

\subsection{Concept Drift Detection}
Concept drift detection is essential in employing a model robustly in data streams \,\cite{lu2018learning, gama2014survey, agrahari2022concept, bayram2022concept}.
Error rate-based drift detection is the most commonly used supervised method for detecting concept drift \,\cite{liu2017fuzzy, bifet2007learning, xu2017dynamic, frias2014online, shao2014prototype}.
This approach continuously monitors the performance of downstream models in data streams. It relies on a trained predictive model and assesses whether concept drift has occurred by examining the consistency of the model's predictive performance over different time intervals\,\cite{bayram2022concept}.
For an unsupervised approach\,\cite{gemaque2020overview}, it is common to conduct statistical tests on the two samples from different periods to determine whether they originate from the same concept \cite{rabanser2019failing, kifer2004detecting, JMLR:v13:gretton12a, 7745962, liu2020learning}, called data distribution-based detection, which is the scope of this work. While some error rate-based detectors\,\cite{raab2020reactive, bifet2007learning} can be adopted for this setting, it is not straightforward to apply them to multivariate data streams. It is also worth noting that some recent works try variants for concept drift detection with a pre-trained model\,\cite{yu2023type, cerqueira2023studd}, active learning\,\cite{yu2021automatic}, imbalanced\,\cite{korycki2021concept} or resource-constrained\,\cite{wang2021clear} streaming settings.

\subsection{Contrastive Learning in Data Streams}

Contrastive learning, as an effective self-supervised learning paradigm \cite{chen2020simple}, is widely applied in various detection tasks in data streams\,\cite{yoon2023scstory, wang2021clear, wang2022online}. The nature of data streams with scarce or delayed labels and lack of external supervision leads to the adoption of continual learning with contrastive losses. The pseudo-labeling for preparing positive and negative samples is a critical design factor and its strategy ranges from model confidence-based\,\cite{yoon2023scstory}, learnable focuses\,\cite{wang2022online}, to class prototype\,\cite{wang2021clear} tailored for downstream tasks. Despite the advancements in contrastive learning, current techniques have yet to be explored for learning separable embeddings of probability distributions representing varying concepts or for application in two-sample tests with statistical bounds, both of which are addressed in this study.

\subsection{Maximun Mean Discrepancy}
The utilization of Maximum Mean Discrepancy (MMD) has been widespread, primarily serving to map data into high-dimensional spaces and thereby enhancing separability for downstream tasks\,\cite{hofmann2008kernel}. MMD has been also actively applied for designing generative models\,\cite{dziugaite2015training, li2017mmd} and detecting whether two samples originate from the same distribution\,\cite{gretton2006kernel} with the Gaussian kernel function\,\cite{JMLR:v13:gretton12a} or the deep kernels\,\cite{liu2020learning} to achieve greater flexibility and expressiveness. The idea of Maximum Concept Discrepancy (MCD) in this study draws inspiration from MMD-based approaches but is specifically tailored for unsupervised online concept drift detection by integrating a deep encoder for sample sets to represent data distributions and continuous learning strategies to dynamically optimize the projected space encompassing varying concepts.

\section{Preliminaries}

\subsection{Concept Drift}
Concept drift is a phenomenon referring to the arbitrary changes in the statistical properties of a target domain of data over time. Formally, concept drift at time $t$ is defined as the change in the joint probability of data points $X$ and labels $y$ at time $t$, denoted as $P_t(X,y) \neq P_{t+1}(X,y)$. Concept drift primarily originates from one of the following three sources\,\cite{lu2018learning}: (i) $P_t(Y|X) \neq P_{t+1}(Y|X)$, when the conditional distribution of the target variable $Y$ given the covariate $X$ undergoes drift; (ii) $P_t(X) \neq P_{t+1}(X)$, when the distribution of the covariate experiences drift; and (iii) combination of (i) and (ii). In addition to varying sources, concept drift can also be distinguished into four types based on the specific nature of the drift occurrence: sudden, reoccurring, gradual, and incremental. For additional references, we direct readers to recent surveys\,\cite{bayram2022concept, agrahari2022concept, lu2018learning}. This work aims to develop an unsupervised method for detecting various types of drift caused by the source described in (ii).

\subsection{Problem Setting}
\label{sec:problem_setting}
\begin{figure}[!t]
    \centering
    \includegraphics[width=\columnwidth]{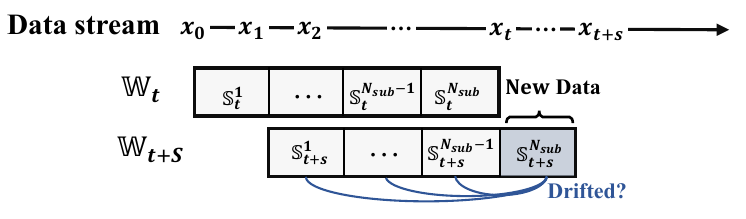}
    \vspace{-0.7cm}
    \caption{Unsupervised online concept drift detection over sliding window $\mathbb{W}$ with sub-windows $\mathbb{S}$.}
    \vspace{-0.5cm}
    \label{fig:problem_setting}
\end{figure} 

Given a continuously evolving data stream $\mathcal{X} = \{ x_t \}_{t=0}^{\infty}$, we maintain the latest context of the data stream by employing a \emph{sliding window} $\mathbb{W}_t$ of size $W$ updated by a slide of size $S$ (i.e., $\mathbb{W}_t = \{x_{t-i} \}_{i=0}^{W-1}$ and $\mathbb{S}_t = \{x_{t-i}\}_{i=0}^{S-1}$). The window and slide sizes can be defined either in terms of the number of data points or a time period. 

Then, a window $\mathbb{W}_t$ consists of non-overlapping slides indexed by $j=1, \ldots, N_{sub}$ where $N_{sub} = W/S$ is the number of slides in a window (i.e., 
$\mathbb{W}_t = \bigcup_{j=1}^{N_{sub}}\mathbb{S}_t^j$ and $\mathbb{S}_t^j=\{x_{t-(N_{sub}-j)*S-i}\}_{i=0}^{S-1}$).
In the rest of the paper, we use the term \emph{sub-window} instead of slide for consistency. It is worth noting that the context within a sub-window is set to be sufficiently compact to ensure that the data points it contains adhere to the same underlying distribution.

For every sliding window in $\mathcal{X}$, the problem of unsupervised online concept drift detection is to identify whether the concept drift has occurred in a new sub-window $\mathbb{S}_t^{N_{sub}}$ compared with the existing data points in the current window $\mathbb{W}_t \setminus \mathbb{S}_t^{N_{sub}}$, without using any labels for drifts and downstream tasks (see Figure \ref{fig:problem_setting}).

\subsection{Maximum Mean Discrepancy}
Maximum Mean Discrepancy (MMD)\,\cite{smola2006maximum} is a statistical measure that compares two probability distributions through a kernel, especially when the distributions are unknown and only samples are available.
MMD evaluates the distance between the mean embeddings of two distributions in the Reproducing Kernel Hilbert Space (RKHS) $\mathcal{H}$\,\cite{berlinet2011reproducing}.
Given $X \sim P$, $Y \sim Q$, and a kernel $f(\cdot)$, we have:
\begin{equation}
    \text{MMD}(P, Q) = \sup_{f \in \mathcal{H}, \|f\|_{\mathcal{H}} \leq 1} \| \mathbb{E}[f(X)] - \mathbb{E}[f(Y)] \|_2.
\end{equation}
When we have samples $\{X_i\}_{i=1}^n$ from probability distribution $P$ and $\{Y_i\}_{i=1}^n$ from probability distribution $Q$, the empirical Maximum Mean Discrepancy (MMD) can be written as:
\begin{equation}
\label{eq:mmd}
    \text{MMD}(P, Q) = \sup_{f \in \mathcal{H}, \|f\|_{\mathcal{H}} \leq 1} \| \frac{1}{n}\sum_{i=1}^nf(X_i) - \frac{1}{n}\sum_{i=1}^nf(Y_i)] \|_2.
\end{equation}

\begin{figure*}[!t]
    \centering
    \includegraphics[width=\textwidth]{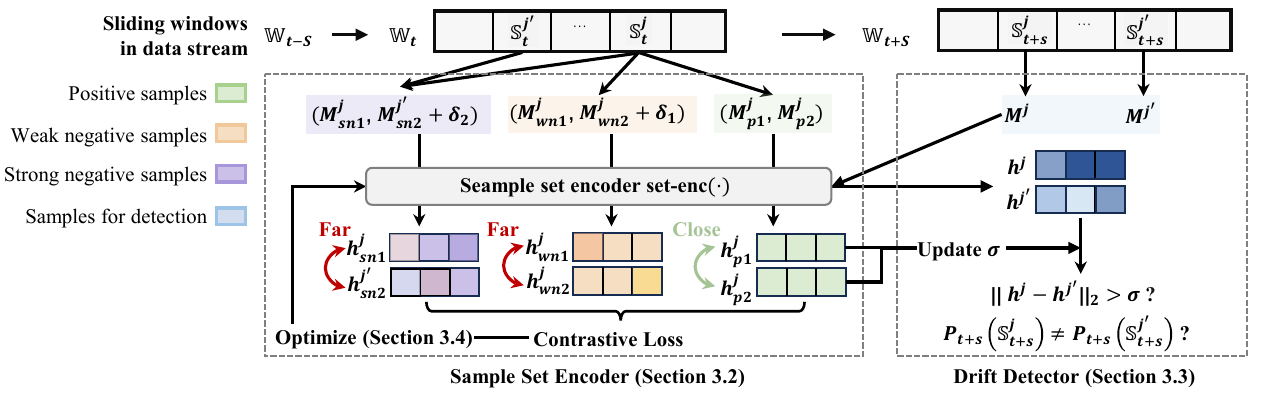}
    \vspace{-0.6cm}
    \caption{Overall procedure of \algname{}. Sub-windows are encoded and compared to derive MCD for drift detection.}
    \vspace{-0.4cm}
    \label{fig:overall_method}
\end{figure*}

\section{Proposed Method}

\subsection{Overview}
The proposed method \algname{} exploits \emph{maximum concept discrepancy} for detecting concept drift, enabled by a deep encoder for data distribution embedding and contrastive learning for effective unsupervised training. The overall procedure of \algname{} is illustrated in Figure \ref{fig:overall_method} and outlined in Algorithm \ref{alg:overall}. For each new sliding window, \algname{} follows the prequential evaluation scheme (i.e., test-and-train)\,\cite{Gam13}. First, \algname{} samples sets of data points in sub-windows and embeds them through a sample set encoder. The discrepancy between sample sets from adjacent sub-windows is calculated, and if it exceeds a threshold, \algname{} asserts that concept drift has occurred between them.
Second, the sample set encoder is updated by considering the sliding window as the context for learning the latest concepts. Specifically, we treat sample sets from the same sub-window as positive pairs and construct negative pairs by leveraging temporal gaps between sub-windows and noise augmentation to distort the original distribution. In the meantime, \algname{} dynamically adjusts the threshold for drift detection in the next sliding window by analyzing the positive sample pairs. Finally, the encoder is updated by aiming to minimize the distance between embeddings of positive sample pairs while simultaneously maximizing the distance between those of negative pairs.

\begin{algorithm}[t!]
\begin{algorithmic}[1]
\caption{\bf Overall Procedure of \algname{}}
\label{alg:overall}
\REQUIRE Data stream $\mathcal{X} = \{ x_t \}_{t=0}^{\infty}$, a sample set encoder $\text{set-enc}()$ with an encoding function $f(\cdot)$.
\STATE Initialize the parameter $\theta$ of $f(\cdot)$ and the MCD threshold $\sigma$
    \FOR{every sliding window $\mathbb{W}_t$ \textbf{in} $\mathcal{X}$} 
        \STATE \algorithmiccomment {{\color{blue} {\sc 1. Drift Detection}}}
        \STATE Obtain sample sets  $M^j$  for $\mathbb{S}_t^{j} \subset \mathbb{W}_t$
        \STATE Obtain concept representations $\{ h^j=\text{set-enc}(M^j)\}_{j=1}^{N_{sub}}$
        \STATE Report drifted \textbf{if} $MCD(P_t(\mathbb{S}_t^{N_{sub}}) , P_t(\mathbb{S}_t^{N_{sub}-1})) > \sigma$
        \STATE \algorithmiccomment {{\color{blue} {\sc 2. Encoder Update}}}
        \STATE Obatin positive, weak/strong negative samples
        \STATE Calculate the loss $\mathcal{L}$ by \cref{loss}
        \STATE Update $\theta$ by $\mathcal{L}$
        \STATE Update $\sigma$ from the positive samples
    \ENDFOR
\end{algorithmic}
\end{algorithm}

\subsection{Sample Set Encoder}

We use a set of data points sampled in each sub-window to estimate its data distribution that represents a concept in the sub-window. Specifically, for a given sub-window $\mathbb{S}_t^j \subset \mathbb{W}_t$, we perform sampling without replacement from it to obtain a \emph{sample set} of size $m$:

\begin{equation}
\small
\label{eq:sample_set}
    M^j = \{x_i \in \mathbb{S}_t^j \}_{i=1}^m.
\end{equation}

The sample set is passed to a sample set encoder, denoted as $\text{set-enc}()$. It translates the probability distribution $P_t(\mathbb{S}_t^j)$ of the sub-window into the compact representation in the projected space, which we call \emph{concept representation} $\textbf{h}^j$ of $\mathbb{S}_t^j$:
\begin{equation}
\small
\label{eq:sample_set_encoder}
    \textbf{h}^j= \text{set-enc}(M^j) = \frac{1}{m}\sum_i^mf(x_i \in M^j ),
\end{equation}
where $f(\cdot)$ is an encoding model with a deep neural network.

\subsection{Drift Detector}
\label{sec:drift_detector}
For each sub-window $\mathbb{S}_t^j$ in a sliding window $\mathbb{W}_t$, we obtain the concept representations $\{ h^j \}_{j=1}^{N_{sub}}$ by the sample set encoder. Then, to effectively quantify the distance between the concept representations from two adjacent sub-windows, we introduce a new measure \emph{Maximum Concept Discrepancy (MCD)} inspired by MMD:

\begin{equation}
\label{eq:mcd}
\small
    \text{MCD}(P, Q) = \sup_{ \|f\|_2 \leq L} \| \mathbb{E}[f(X)] - \mathbb{E}[f(Y)] \|_2,
\end{equation}
where the function $f(\cdot)$ is approximated by a deep neural network and is constrained to be Lipschitz continuous\,\cite{gouk2021regularisation}. This ensures the convergence of optimizing $f$ and prevents MCD from becoming infinitely large. Moreover, it makes MCD between two sets of independent data from the same distribution bounded, providing the theoretical foundation for our drift detection.

When we have samples $M^j \sim P$ and $M^{j'} \sim Q$ by \cref{eq:sample_set}, the empirical MCD can be derived from \cref{eq:mcd} and \cref{eq:sample_set_encoder}  as:

\begin{equation}
\small
\begin{aligned}
    \text{MCD}(P, Q) &= \sup_{\|f\|_2 \leq L} \| \frac{1}{m}\sum_{i=1}^mf(x_i \in M^j) - \frac{1}{m}\sum_{i=1}^mf(x_i \in M^{j'})] \|_2. \\
     &= \sup_{\|f\|_2 \leq L} \| \textbf{h}^{j} - \textbf{h}^{j'} \|_2
\end{aligned}
\end{equation}

Based on MCD, if the discrepancy between the data distributions of two adjacent sub-windows exceeds a given threshold $\sigma$,
\begin{equation}
\small
    \text{MCD}(P_t(\mathbb{S}_t^{j}) , P_t(\mathbb{S}_t^{j-1})) = \| \textbf{h}^{j} - \textbf{h}^{j-1} \|_2 > \sigma,
\end{equation}
then we determine that a concept drift happened between them:
\begin{equation}
\small
    P_t(\mathbb{S}_t^{j}) \neq P_t(\mathbb{S}_t^{j-1}).
\end{equation}

The threshold $\sigma$ can be controlled dynamically using a bootstrapping strategy, allowing it to adjust to the varying distances between sample sets within the same concept. By analyzing the historical collection of MCD values between sample sets in the same sub-windows (i.e., those with identical data distributions), we adjust the threshold $\sigma$ for each sliding window within a predefined statistical significance level (e.g., 0.05). It is worth noting that this historical MCD information can be obtained during the optimization of the encoder without necessitating additional computations.

\subsection{Optimization}
To optimize the sample set encoder, we employ contrastive learning\,\cite{he2020momentum} to enhance the separability of various concepts. The primary challenge lies in determining positive and negative sample pairs that are to be closer and further apart, respectively. This challenge is particularly pronounced in unsupervised scenarios, where we lack any prior information regarding the underlying concepts and true drifts. In \algname{}, we exploit the concept of temporal coherence\,\cite{shin2021coherence, wang2020boundary}, which is a fundamental characteristic observed in temporal data. It suggests that data points that are close in time are more likely to exhibit similar characteristics. We exploit this insight to create positive and negative samples used for optimization.

\subsubsection{\textbf{Preparing Positive Samples}}
\algname{} generates a positive sample pair by choosing two sets of data points from the same sub-window, assuming that these sets follow the same distributions.

In each sub-window $\mathbb{S}t^j$, \algname{} conducts sampling similar to \cref{eq:sample_set} to select two sets of data points, which we treat as positive sample pairs. These pairs are denoted as $M_{p1}^j$ and $M_{p2}^{j}$. To generate a diverse set of positive sample pairs, we repeat this sampling process $k$ times, resulting in $\{ M_{p1}^{j,i} \}_{i=1}^k$ and $\{ M_{p2}^{j,i} \}_{i=1}^k$. Subsequently, the sample set encoder in \cref{eq:sample_set_encoder} computes the embeddings for the $k$ positive sample pairs, yielding $\{\textbf{h}_{p1}^{j,i}\}_{i=1}^k$ and $\{\textbf{h}_{p2}^{j,i}\}_{i=1}^k$.

\subsubsection{\textbf{Preparing Negative Samples}}
\algname{} generates a negative sample pair to learn diversity in concept drift. Specifically, \algname{} selects two sets of data points respectively from the different sub-windows that are temporally distant, whose distributions are likely to differ. \algname{} also employs an efficient data augmentation technique to improve the generalization of negative pairs by introducing noise to each sampled data point:
\begin{equation}
\small
    x'=x+\delta,
\end{equation}
where the noise $\delta$ is generated from a standard probability distribution (e.g., the Gaussian distribution $G(\mu, \epsilon$)). By incorporating these two principles, \emph{temporal gap} (i.e., drift diversity) and \emph{noise augmentation} (i.e., concept diversity), we prepare two types of negative sample pairs.

\underline{Weak negative samples}: The first type of negative samples involves sampling pairs of data points from the same sub-windows but adding a small degree of noise into one of the sets. Specifically, given a sub-window $\mathbb{S}t^j$, the $k$ weak negative sample pairs are prepared: 
\begin{equation}
\small
\{ M_{wn1}^{j,i} \}_{i=1}^k \text{ and } \{ M_{wn2}^{j,i} + \delta_1 \}_{i=1}^k,
\end{equation}
where $\delta_1 \sim G(0, \epsilon_{small})$
and $G$ is a Gaussian distribution. Finally, the corresponding two set sample embeddings are derived: 
\begin{equation}
\small
\{\textbf{h}_{wn1}^{j,i}\}_{i=1}^k \text{ and } \{\textbf{h}_{wn2}^{j,i}\}_{i=1}^k.
\end{equation}

\underline{Strong negative samples}:  The second type of negative samples involves sampling pairs of data points from the two sub-windows that are temporally distant and more substantial noise is added to one of the sets. Specifically, given a sub-window $\mathbb{S}t^j$ and $\mathbb{S}t^{j'}$, the $k$ strong negative sample pairs are prepared: 
\begin{equation}
\{ M_{sn1}^{j,i} \}_{i=1}^k \text{ and } \{ M_{sn2}^{j',i} + \delta_2 \}_{i=1}^k,
\end{equation}

where $\delta_2 \sim G(0, \epsilon_{big})$. Similarly, the corresponding two set sample embeddings are derived: 
\begin{equation}
\{\textbf{h}_{sn1}^{j,i}\}_{i=1}^k \text{ and } \{\textbf{h}_{sn2}^{j',i}\}_{i=1}^k.
\end{equation}

Note that the temporal gap between two sub-windows can be adjusted within the context of a window (i.e., $1 \leq |j-j'| \leq N_{sub}-1$). By default, \algname{} adopts the largest temporal gap by setting $j = N_{sub}$ and $j' = 1$ so as to maximize the likelihood that the samples within each window exhibit distinct data distributions.

\subsubsection{\textbf{Learning Objective}}
By putting the positive, weak negative, and strong negative samples altogether, the final loss for the sample set encoder is formulated in the form of InfoNCE loss\,\cite{oord2018representation}:
\begin{equation}
\label{loss}
\small
\begin{aligned}
\mathcal{L}=\log\sum_{j=1}^{N_{sub}}\frac{
{\sum_{i=1}^k\exp(\text{MCD}_{p}^{j,k})}
}{
\sum_{i=1}^k(
\exp(\text{MCD}_{p}^{j,k})
+
\exp(\text{MCD}_{wn}^{j,k})
+
\exp(\text{MCD}_{sn}^{j,k})
)}
    \\
+\lambda \sum_{i=1}^m(||\nabla _{x_i}f(x_i)||_2-L)^2,
\end{aligned}
\end{equation}
where $\text{MCD}_{p}^k = ||\textbf{h}_{p1}^{j,k}-\textbf{h}_{p2}^{j,k}||_2$, $\text{MCD}_{wn}^k = ||\textbf{h}_{wn1}^{j,k}-\textbf{h}_{wn2}^{j,k}||_2$, and $\text{MCD}_{sn}^k = ||\textbf{h}_{sn1}^{j,k}-\textbf{h}_{sn2}^{j',k}||_2$ are MCD values between the probability distributions of two sub-windows chosen for each sample type. The last term is the gradient penalty to ensure the L-Lipschitz continuity\,\cite{eriksson2004lipschitz} of the encoding model $f(\cdot)$ where $L$ is a constant and the coefficient $\lambda$ is the regularization parameter.



\section{Theoretical Analysis} 
\subsection{Upper Bound of MCD}
Given two sample sets of data points drawn from the same distribution, we study the upper bound of MCD between the two sets, which can serve as a theoretical threshold for detecting concept drifts in two sub-windows for \algname{}.

\bgroup
\def\arraystretch{0.9}%
\begin{table*}[t!]
\centering
\caption{Description of data streams used for evaluating drift detection performance.}
\vspace{-0.3cm}
\small
\label{tbl:datasets}
\begin{tabular}{cccccc} \toprule
Stream Category & Data Set & Composition & Instances & Dimension & Drift Type(s) \\
\midrule
\multirow{4}{*}{Synthetic (Primary)} 
& GM\_Sud & Gaussian Mixture & 30,000 & 5 & Sudden \\
& GM\_Rec & Gaussian Mixture & 30,000 & 5 & Reoccurring  \\
& GM\_Grad & Gaussian Mixture & 30,000 & 5 & Gradual  \\
& GM\_Inc & Gaussian Mixture & 30,000 & 5 & Incremental  \\
\hdashline
\multirow{3}{*}{Synthetic (Complex)} 
& GamLog\_Sud & Gamma, Lognormal & 30,000 & 5 & Sudden  \\
& LogGamWei\_Sud & Lognormal, Gamma, Weibull  & 30,000 & 20 & Sudden 
\\
& GamGM\_SudGrad &Gamma,Gaussian Mixture& 30,000 & 20 & Sudden, Gradual \\
\hline
\multirow{4}{*}{Real-World} 
& INSECTS\_Sud & Mosquito sensors with varying temperatures & 52,848 & 33 & Sudden \\
& INSECTS\_Grad & Mosquito sensors with varying temperatures &  24,150 & 33 & Gradual \\
& INSECTS\_IncreRec & Mosquito sensors with varying temperatures &  79,986 & 33 & Incremental, Reoccurring \\
& EEG & EEG readings for open/closed eye states & 14,980 & 14 & Sudden, Reoccurring \\
\bottomrule
\end{tabular}
\vspace{-0.1cm}
\end{table*}
\egroup

\begin{theorem}
\label{bound1}
Assume that the sets $ \{X_i\}_{i=1}^n $ and $ \{Y_i\}_{i=1}^n $ are independently and identically distributed (i.i.d.), both drawn from the probability distribution $ p(x) $ with a mean $ \mu $ and variance $ \sigma $. If $ f $ is a Lipschitz continuous function with Lipschitz constant $ L $, we have:
\begin{equation}
\small
\label{bound}
    P(|\frac{1}{n}\sum_{i=1}^nf(X_i)-\frac{1}{n}\sum_{i=1}^nf(Y_i)|)>G(1-\frac{\alpha}{2})\sqrt{\frac{2}{n}}L\sigma)\leq\alpha,
\end{equation}
where $G(\cdot)$ is the standard Gaussian distribution function.
Therefore, for a given significance level $\alpha$, $G(1-\frac{\alpha}{2})\sqrt{\frac{2}{n}}L\sigma$ is an upper bound for the MCD $|\frac{1}{n}\sum_{i=1}^nf(X_i)-\frac{1}{n}\sum_{i=1}^nf(Y_i)|$.

\end{theorem}
\begin{proof}
    By the Central Limit Theorem, $\frac{1}{n}\sum_{i=1}^nf(X_i)-\frac{1}{n}\sum_{i=1}^nf(Y_i) $ converges to the Guassain distribution.
    Moreover, we have:
    \begin{equation}
    \small
        E(\frac{1}{n}\sum_{i=1}^nf(X_i)-\frac{1}{n}\sum_{i=1}^nf(Y_i))=0.
    \end{equation}
    Since $ f $ is a $L$-Lipschitz continuous function, we have:
    \begin{equation}
    \small
        Var(\frac{1}{n}\sum_{i=1}^nf(X_i)-\frac{1}{n}\sum_{i=1}^nf(Y_i)) \leq \frac{2L^2\sigma^2}{n}.
    \end{equation}
When we approximate the distribution of 
    $\frac{1}{n}\sum_{i=1}^nf(X_i)-\frac{1}{n}\sum_{i=1}^nf(Y_i) $ as the Gaussian distribution, we have:
    \begin{equation}
    \small
        P(|\frac{1}{n}\sum_{i=1}^nf(X_i)-\frac{1}{n}\sum_{i=1}^nf(Y_i)|>G(1-\frac{\alpha}{2})\sqrt{\frac{2}{n}}L\sigma)\leq\alpha.
    \end{equation}
    
\end{proof}


Therefore, for the null hypothesis $H_0$: $ \{X_i\}_{i=1}^n $ and $ \{Y_i\}_{i=1}^n $ are drawn from the same probability distribution, given a significance level $\alpha$, $G(1-\frac{\alpha}{2})\sqrt{\frac{2}{n}}L\sigma$ can serve as the threshold for rejecting \(H_0\).
In the scenario where $ \{X_i\}_{i=1}^n $ and $ \{Y_i\}_{i=1}^n $ are multivariate random variables, hypothesis testing can similarly be conducted using the chi-squared distribution.

This theoretical bound can serve as a guide to set the threshold in a hypothesis testing framework. However, deriving the exact rejection threshold analytically may not always be feasible. As suggested in Section \ref{sec:drift_detector}, the empirical threshold for rejecting the null hypothesis can be used by estimating statistics of historical MCD values meeting the hypothesis with a pre-defined significance.

\subsection{Complexity of \algname{}}
We analyze the time complexity of \algname{} mainly for sampling, training, and inference. Recall that we use a sliding window with $N_{sub}$ sub-windows, $k$ sets of $m$ samples, and an encoder with the parameter size $p$ and training epochs $e$. Since we sample in each sub-window, the time complexity for constructing positive and negative samples is $\mathcal{O}(mkN_{sub})$. The time complexity for training the encoder is $\mathcal{O}(mkep)$.
For inference, since we need to calculate the differences of sample sets in each sub-window sequentially, the time complexity is $O(mkN_{sub}^2)$. Finally, the total complexity is $O(mk(N_{sub}^2+ep))$. Since typically $p \gg e,m,k,N_{sub}$, the time complexity of \algname{} is mostly controlled by the encoder complexity.

\section{Experiments}
We conducted thorough experiments to evaluate the performance of \algname{} on 7 synthetic data sets and 4 real-world data sets. The results are briefly summarized as follows.
\begin{itemize}[leftmargin=10pt, noitemsep]
\item \algname{} outperforms existing baselines in detecting concept drifts in terms of Precision, F1, and MCC scores and shows high interpretablity with varying drift types (Section \ref{sec:overall_performance}).
\item Through ablation analysis, the three sampling strategies introduced in contrastive learning for MCD are demonstrated to be effective (Section \ref{sec:ablation}).
\item The hyperparameters for \algname{} are reasonably set by default and robust to the performance in most cases (Section \ref{sec:sensitivity_analysis}).
\item Visualizations of concept embeddings (Section \ref{sec:qualitative}) and threshold changes (Section \ref{sec:threshold}) show the  empirical efficacy of \algname{}.
\end{itemize}

\bgroup
\def\arraystretch{0.8}%
\begin{table*}[htbp]
\centering
\small
\caption{Overall performance comparison (the best and second best results are in bold and underlined, respectively).}
\vspace{-0.3cm}
\label{tbl:overall_performance}
\begin{tabular}[c]{@{}C{2.0cm}|C{0.65cm}C{0.65cm}C{0.65cm}|C{0.65cm}C{0.65cm}C{0.65cm}|C{0.65cm}C{0.65cm}C{0.65cm}|C{0.65cm}C{0.65cm}C{0.65cm}|C{0.65cm}C{0.65cm}C{0.65cm}@{}}
\toprule
\multicolumn{1}{c}{}& \multicolumn{3}{c}{\textbf{\algname{}}} & \multicolumn{3}{c}{\textbf{KS}} & \multicolumn{3}{c}{\textbf{MMD-GK}} & \multicolumn{3}{c}{\textbf{LSDD}} & \multicolumn{3}{c}{\textbf{MMD-DK}} \\ 
\multicolumn{1}{c}{\textit{Data set}} & \multicolumn{1}{c}{\textit{Pre.}} & \multicolumn{1}{c}{\textit{F1}} & \multicolumn{1}{c}{\textit{MCC}} & \multicolumn{1}{c}{\textit{Pre.}} & \multicolumn{1}{c}{\textit{F1}} & \multicolumn{1}{c}{\textit{MCC}} & \multicolumn{1}{c}{\textit{Pre.}} & \multicolumn{1}{c}{\textit{F1}} & \multicolumn{1}{c}{\textit{MCC}} & \multicolumn{1}{c}{\textit{Pre.}} & \multicolumn{1}{c}{\textit{F1}} & \multicolumn{1}{c}{\textit{MCC}} & \multicolumn{1}{c}{\textit{Pre.}} & \multicolumn{1}{c}{\textit{F1}}& \multicolumn{1}{c}{\textit{MCC}}\\ \toprule
 
\begin{tabular}[c]{@{}c@{}}GM\_Sud\end{tabular} & \begin{tabular}[c]{@{}c@{}}\textbf{1.00}\\ ($\pm$0.00)\end{tabular} & \begin{tabular}[c]{@{}c@{}}\textbf{1.00}\\ ($\pm$0.00)\end{tabular} & \begin{tabular}[c]{@{}c@{}}\textbf{1.00}\\ ($\pm$0.00)\end{tabular} & \begin{tabular}[c]{@{}c@{}}0.25\\ ($\pm$0.00)\end{tabular}& \begin{tabular}[c]{@{}c@{}}0.40\\ ($\pm$0.00)\end{tabular}& \begin{tabular}[c]{@{}c@{}}0.49\\ ($\pm$0.00)\end{tabular} & \begin{tabular}[c]{@{}c@{}}\underline{0.32}\\ ($\pm$0.07)\end{tabular} &\begin{tabular}[c]{@{}c@{}}\underline{0.48}\\ ($\pm$0.08)\end{tabular}& \begin{tabular}[c]{@{}c@{}}\underline{0.56}\\ ($\pm$0.06)\end{tabular}&
\begin{tabular}[c]{@{}c@{}}0.17\\ ($\pm$0.04)\end{tabular} &\begin{tabular}[c]{@{}c@{}}0.30\\ ($\pm$0.05)\end{tabular}& \begin{tabular}[c]{@{}c@{}}0.40\\ ($\pm$0.05)\end{tabular}&
\begin{tabular}[c]{@{}c@{}}0.25\\($\pm$0.12)\end{tabular}& \begin{tabular}[c]{@{}c@{}}0.38\\ ($\pm$0.15)\end{tabular}& \begin{tabular}[c]{@{}c@{}}0.47\\ ($\pm$0.12)\end{tabular}\\
\hdashline
GM\_Rec & 
\begin{tabular}[c]{@{}c@{}}\textbf{1.00}\\ ($\pm$0.00)\end{tabular} & \begin{tabular}[c]{@{}c@{}}\textbf{0.79}\\ ($\pm$0.11)\end{tabular} & \begin{tabular}[c]{@{}c@{}}\textbf{0.81}\\ ($\pm$0.10)\end{tabular} & \begin{tabular}[c]{@{}c@{}}0.40\\ ($\pm$0.00)\end{tabular} & \begin{tabular}[c]{@{}c@{}}0.50\\ ($\pm$0.00)\end{tabular} & \begin{tabular}[c]{@{}c@{}}0.50\\ ($\pm$0.00)\end{tabular} & \begin{tabular}[c]{@{}c@{}}\underline{0.64}\\ ($\pm$0.08)\end{tabular} & \begin{tabular}[c]{@{}c@{}}\underline{0.78}\\ ($\pm$0.06)\end{tabular}& \begin{tabular}[c]{@{}c@{}}\underline{0.79}\\ ($\pm$0.05)\end{tabular}& \begin{tabular}[c]{@{}c@{}}0.60\\ ($\pm$0.16)\end{tabular}& \begin{tabular}[c]{@{}c@{}}0.74\\ ($\pm$0.12)\end{tabular}& \begin{tabular}[c]{@{}c@{}}0.76\\ ($\pm$0.10)\end{tabular}& \begin{tabular}[c]{@{}c@{}}0.37\\ ($\pm$0.11)\end{tabular}& \begin{tabular}[c]{@{}c@{}}0.51\\ ($\pm$0.12)\end{tabular}& \begin{tabular}[c]{@{}c@{}}0.54\\ ($\pm$0.13)\end{tabular}\\[2pt]
\hdashline
GM\_Grad & 
\begin{tabular}[c]{@{}c@{}}\textbf{1.00}\\ ($\pm$0.00)\end{tabular} & \begin{tabular}[c]{@{}c@{}}\underline{0.79}\\ ($\pm$0.06)\end{tabular} & \begin{tabular}[c]{@{}c@{}}\underline{0.79}\\ ($\pm$0.06)\end{tabular} & \begin{tabular}[c]{@{}c@{}}0.78\\ ($\pm$0.00)\end{tabular} & \begin{tabular}[c]{@{}c@{}}0.78\\ ($\pm$0.00)\end{tabular}& \begin{tabular}[c]{@{}c@{}}0.76\\ ($\pm$0.00)\end{tabular}& \begin{tabular}[c]{@{}c@{}}\underline{0.80}\\ ($\pm$0.05)\end{tabular}& \begin{tabular}[c]{@{}c@{}}\textbf{0.89}\\ ($\pm$0.03)\end{tabular}& \begin{tabular}[c]{@{}c@{}}\textbf{0.88}\\ ($\pm$0.03)\end{tabular}& \begin{tabular}[c]{@{}c@{}}0.79\\ ($\pm$0.06)\end{tabular}& \begin{tabular}[c]{@{}c@{}}0.78\\ ($\pm$0.03)\end{tabular}& \begin{tabular}[c]{@{}c@{}}0.77\\ ($\pm$0.03)\end{tabular}& \begin{tabular}[c]{@{}c@{}}0.70\\ ($\pm$0.11)\end{tabular}& \begin{tabular}[c]{@{}c@{}}\underline{0.79}\\ ($\pm$0.08)\end{tabular}& \begin{tabular}[c]{@{}c@{}}0.78\\ ($\pm$0.09)\end{tabular}\\[2pt]
\hdashline
GM\_Inc & 
\begin{tabular}[c]{@{}c@{}}\textbf{0.99}\\ ($\pm$0.04)\end{tabular} & \begin{tabular}[c]{@{}c@{}}0.44\\ ($\pm$0.09)\end{tabular} & \begin{tabular}[c]{@{}c@{}}0.51\\ ($\pm$0.07)\end{tabular} & \begin{tabular}[c]{@{}c@{}}0.38\\ ($\pm$0.00)\end{tabular}&
\begin{tabular}[c]{@{}c@{}}0.33\\ ($\pm$0.00)\end{tabular}& \begin{tabular}[c]{@{}c@{}}0.27\\ ($\pm$0.00)\end{tabular}& \begin{tabular}[c]{@{}c@{}}\underline{0.64}\\ ($\pm$0.05)\end{tabular}& \begin{tabular}[c]{@{}c@{}}\textbf{0.71}\\ ($\pm$0.03)\end{tabular}& \begin{tabular}[c]{@{}c@{}}\textbf{0.67}\\ ($\pm$0.04)\end{tabular}& \begin{tabular}[c]{@{}c@{}}0.57\\ ($\pm$0.06)\end{tabular}& \begin{tabular}[c]{@{}c@{}}\underline{0.65}\\ ($\pm$0.05)\end{tabular}& \begin{tabular}[c]{@{}c@{}}\underline{0.61}\\ ($\pm$0.06)\end{tabular}& \begin{tabular}[c]{@{}c@{}}0.44\\ ($\pm$0.18)\end{tabular}& \begin{tabular}[c]{@{}c@{}}0.37\\ ($\pm$0.13)\end{tabular}& \begin{tabular}[c]{@{}c@{}}0.31\\ ($\pm$0.15)\end{tabular}\\[2pt]
\hline
GamLog\_Sud & 
\begin{tabular}[c]{@{}c@{}}\textbf{1.00}\\ ($\pm$0.00)\end{tabular} & \begin{tabular}[c]{@{}c@{}}\textbf{1.00}\\ ($\pm$0.00)\end{tabular} & \begin{tabular}[c]{@{}c@{}}\textbf{1.00}\\ ($\pm$0.00)\end{tabular} & \begin{tabular}[c]{@{}c@{}}0.13\\ ($\pm$0.00)\end{tabular}& \begin{tabular}[c]{@{}c@{}}0.22\\ ($\pm$0.00)\end{tabular}& \begin{tabular}[c]{@{}c@{}}0.34\\ ($\pm$0.00)\end{tabular}& \begin{tabular}[c]{@{}c@{}}0.11\\ ($\pm$0.01)\end{tabular}& \begin{tabular}[c]{@{}c@{}}0.20\\ ($\pm$0.02)\end{tabular}& \begin{tabular}[c]{@{}c@{}}0.31\\ ($\pm$0.02)\end{tabular}& \begin{tabular}[c]{@{}c@{}}0.11\\ ($\pm$0.02)\end{tabular}& \begin{tabular}[c]{@{}c@{}}0.20\\ ($\pm$0.03)\end{tabular}& \begin{tabular}[c]{@{}c@{}}0.32\\ ($\pm$0.03)\end{tabular}& \begin{tabular}[c]{@{}c@{}}\underline{0.18}\\ ($\pm$0.06)\end{tabular}& \begin{tabular}[c]{@{}c@{}}\underline{0.30}\\ ($\pm$0.08)\end{tabular}& \begin{tabular}[c]{@{}c@{}}\underline{0.40}\\ ($\pm$0.07)\end{tabular}\\
\hdashline
LogGamWei\_Sud & 
\begin{tabular}[c]{@{}c@{}}\textbf{0.98}\\ ($\pm$0.07)\end{tabular} & \begin{tabular}[c]{@{}c@{}}\textbf{0.94}\\ ($\pm$0.12)\end{tabular} & \begin{tabular}[c]{@{}c@{}}\textbf{0.95}\\ ($\pm$0.11)\end{tabular} & 
\begin{tabular}[c]{@{}c@{}}\underline{0.40}\\ ($\pm$0.00)\end{tabular}& \begin{tabular}[c]{@{}c@{}}\underline{0.57}\\ ($\pm$0.00)\end{tabular}& \begin{tabular}[c]{@{}c@{}}\underline{0.62}\\ ($\pm$0.00)\end{tabular}& \begin{tabular}[c]{@{}c@{}}0.31\\ ($\pm$0.06)\end{tabular}& \begin{tabular}[c]{@{}c@{}}0.46\\ ($\pm$0.06)\end{tabular}& \begin{tabular}[c]{@{}c@{}}0.54\\ ($\pm$0.05)\end{tabular}& \begin{tabular}[c]{@{}c@{}}0.26\\ ($\pm$0.10)\end{tabular}& \begin{tabular}[c]{@{}c@{}}0.41\\ ($\pm$0.10)\end{tabular}& \begin{tabular}[c]{@{}c@{}}0.49\\ ($\pm$0.09)\end{tabular}& \begin{tabular}[c]{@{}c@{}}0.29\\ ($\pm$0.07)\end{tabular}& \begin{tabular}[c]{@{}c@{}}0.44\\ ($\pm$0.08)\end{tabular}& \begin{tabular}[c]{@{}c@{}}0.52\\ ($\pm$0.06)\end{tabular}\\
\hdashline
GamGM\_SudGrad & 
\begin{tabular}[c]{@{}c@{}}\textbf{0.98}\\ ($\pm$0.04)\end{tabular} & \begin{tabular}[c]{@{}c@{}}\textbf{0.99}\\ ($\pm$0.02)\end{tabular} & \begin{tabular}[c]{@{}c@{}}\textbf{0.99}\\ ($\pm$0.02)\end{tabular} & \begin{tabular}[c]{@{}c@{}}\underline{0.88}\\ ($\pm$0.00)\end{tabular}& \begin{tabular}[c]{@{}c@{}}\underline{0.93}\\ ($\pm$0.00)\end{tabular}& \begin{tabular}[c]{@{}c@{}}\underline{0.93}\\ ($\pm$0.00)\end{tabular}& \begin{tabular}[c]{@{}c@{}}0.66\\ ($\pm$0.05)\end{tabular}& \begin{tabular}[c]{@{}c@{}}0.79\\ ($\pm$0.04)\end{tabular}& \begin{tabular}[c]{@{}c@{}}0.79\\ ($\pm$0.03)\end{tabular}& \begin{tabular}[c]{@{}c@{}}0.57\\ ($\pm$0.07)\end{tabular}& \begin{tabular}[c]{@{}c@{}}0.72\\ ($\pm$0.06)\end{tabular}& \begin{tabular}[c]{@{}c@{}}0.73\\ ($\pm$0.05)\end{tabular}& \begin{tabular}[c]{@{}c@{}}0.61\\ ($\pm$0.11)\end{tabular}& \begin{tabular}[c]{@{}c@{}}0.75\\ ($\pm$0.09)\end{tabular}& \begin{tabular}[c]{@{}c@{}}0.76\\ ($\pm$0.08)\end{tabular}\\
\hline
INSECTS\_Sud & 
\begin{tabular}[c]{@{}c@{}}\textbf{0.55}\\ ($\pm$0.06)\end{tabular} &
\begin{tabular}[c]{@{}c@{}}\textbf{0.59}\\ ($\pm$0.07)\end{tabular} &
\begin{tabular}[c]{@{}c@{}}\textbf{0.55}\\ ($\pm$0.08)\end{tabular} &
\begin{tabular}[c]{@{}c@{}}\underline{0.38}\\ ($\pm$0.00)\end{tabular} &
\begin{tabular}[c]{@{}c@{}}\underline{0.53}\\ ($\pm$0.00)\end{tabular} &
\begin{tabular}[c]{@{}c@{}}\underline{0.53}\\ ($\pm$0.00)\end{tabular}& \begin{tabular}[c]{@{}c@{}}\underline{0.38}\\ ($\pm$0.02)\end{tabular} &
\begin{tabular}[c]{@{}c@{}}\underline{0.53}\\ ($\pm$0.02)\end{tabular} &
\begin{tabular}[c]{@{}c@{}}\underline{0.53}\\ ($\pm$0.02)\end{tabular} & \begin{tabular}[c]{@{}c@{}}0.33\\ ($\pm$0.02)\end{tabular} &
\begin{tabular}[c]{@{}c@{}}0.49\\ ($\pm$0.02)\end{tabular} &
\begin{tabular}[c]{@{}c@{}}0.48\\ ($\pm$0.02)\end{tabular} & \begin{tabular}[c]{@{}c@{}}\underline{0.38}\\ ($\pm$0.06)\end{tabular} &
\begin{tabular}[c]{@{}c@{}}0.51\\ ($\pm$0.07)\end{tabular} &
\begin{tabular}[c]{@{}c@{}}0.48\\ ($\pm$0.08)\end{tabular} \\
\hdashline
INSECTS\_Grad & 
\begin{tabular}[c]{@{}c@{}}\textbf{0.18}\\ ($\pm$0.15)\end{tabular} &
\begin{tabular}[c]{@{}c@{}}\textbf{0.28}\\ ($\pm$0.23)\end{tabular} &
\begin{tabular}[c]{@{}c@{}}\textbf{0.32}\\ ($\pm$0.28)\end{tabular} &
\begin{tabular}[c]{@{}c@{}}\underline{0.06}\\ ($\pm$0.00)\end{tabular} &
\begin{tabular}[c]{@{}c@{}}0.11\\ ($\pm$0.00)\end{tabular} &
\begin{tabular}[c]{@{}c@{}}0.22\\ ($\pm$0.00)\end{tabular}& \begin{tabular}[c]{@{}c@{}}\underline{0.06}\\ ($\pm$0.00)\end{tabular} &
\begin{tabular}[c]{@{}c@{}}\underline{0.12}\\ ($\pm$0.01)\end{tabular} &
\begin{tabular}[c]{@{}c@{}}\underline{0.23}\\ ($\pm$0.01)\end{tabular} & \begin{tabular}[c]{@{}c@{}}0.05\\ ($\pm$0.01)\end{tabular} &
\begin{tabular}[c]{@{}c@{}}0.10\\ ($\pm$0.01)\end{tabular} &
\begin{tabular}[c]{@{}c@{}}0.21\\ ($\pm$0.01)\end{tabular} & \begin{tabular}[c]{@{}c@{}}\underline{0.06}\\ ($\pm$0.04)\end{tabular} &
\begin{tabular}[c]{@{}c@{}}0.11\\ ($\pm$0.08)\end{tabular} &
\begin{tabular}[c]{@{}c@{}}0.18\\ ($\pm$0.14)\end{tabular} \\
\hdashline
INSECTS\_IncreRec & 
\begin{tabular}[c]{@{}c@{}}\textbf{0.26}\\ ($\pm$0.05)\end{tabular} &
\begin{tabular}[c]{@{}c@{}}\textbf{0.37}\\ ($\pm$0.07)\end{tabular} &
\begin{tabular}[c]{@{}c@{}}\textbf{0.40}\\ ($\pm$0.07)\end{tabular} &
\begin{tabular}[c]{@{}c@{}}0.08\\ ($\pm$0.00)\end{tabular} &
\begin{tabular}[c]{@{}c@{}}0.15\\ ($\pm$0.00)\end{tabular} &
\begin{tabular}[c]{@{}c@{}}0.24\\ ($\pm$0.00)\end{tabular}& \begin{tabular}[c]{@{}c@{}}0.08\\ ($\pm$0.00)\end{tabular} &
\begin{tabular}[c]{@{}c@{}}0.15\\ ($\pm$0.00)\end{tabular} &
\begin{tabular}[c]{@{}c@{}}0.24\\ ($\pm$0.00)\end{tabular} & \begin{tabular}[c]{@{}c@{}}0.08\\ ($\pm$0.00)\end{tabular} &
\begin{tabular}[c]{@{}c@{}}0.14\\ ($\pm$0.00)\end{tabular} &
\begin{tabular}[c]{@{}c@{}}0.23\\ ($\pm$0.00)\end{tabular} & \begin{tabular}[c]{@{}c@{}}\underline{0.10}\\ ($\pm$0.01)\end{tabular} &
\begin{tabular}[c]{@{}c@{}}\underline{0.18}\\ ($\pm$0.02)\end{tabular} &
\begin{tabular}[c]{@{}c@{}}\underline{0.25}\\ ($\pm$0.04)\end{tabular} \\
\hdashline
EEG & 
\begin{tabular}[c]{@{}c@{}}0.43\\ ($\pm$0.14)\end{tabular} & \begin{tabular}[c]{@{}c@{}}0.23\\ ($\pm$0.10)\end{tabular} & \begin{tabular}[c]{@{}c@{}}\textbf{0.12}\\ ($\pm$0.09)\end{tabular} & \begin{tabular}[c]{@{}c@{}}0.25\\ ($\pm$0.00)\end{tabular} & \begin{tabular}[c]{@{}c@{}}0.40\\ ($\pm$0.00)\end{tabular} & \begin{tabular}[c]{@{}c@{}}-0.05\\ ($\pm$0.00)\end{tabular} & \begin{tabular}[c]{@{}c@{}}\underline{0.47}\\ ($\pm$0.00)\end{tabular} & \begin{tabular}[c]{@{}c@{}}\underline{0.63}\\ ($\pm$0.00)\end{tabular} & \begin{tabular}[c]{@{}c@{}}-0.11\\ ($\pm$0.00)\end{tabular} & \begin{tabular}[c]{@{}c@{}}0.25\\ ($\pm$0.00)\end{tabular} & \begin{tabular}[c]{@{}c@{}}0.40\\ ($\pm$0.00)\end{tabular} & \begin{tabular}[c]{@{}c@{}}-0.00\\ ($\pm$0.00)\end{tabular} & \begin{tabular}[c]{@{}c@{}}\textbf{0.48}\\ ($\pm$0.00)\end{tabular} & \begin{tabular}[c]{@{}c@{}}\textbf{0.64}\\ ($\pm$0.00)\end{tabular} & \begin{tabular}[c]{@{}c@{}}\underline{0.03}\\ ($\pm$0.04)\end{tabular} \\
\bottomrule
\end{tabular}
\end{table*}
\egroup
\begin{figure*}[!t]
    \centering
    \begin{subfigure}[b]{0.49\textwidth}
        \centering
        \includegraphics[width=\textwidth]{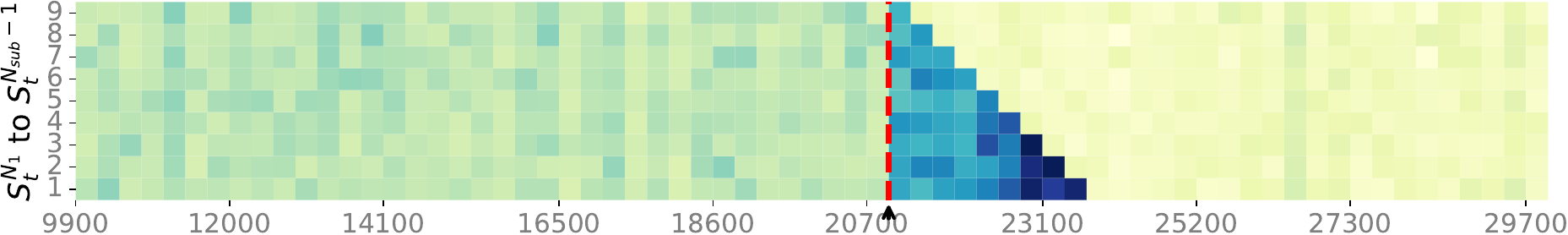}
        \vspace{-0.5cm}
        \caption{GM\_Sud.}
        \label{fig:hm_sud}
    \end{subfigure}
    \begin{subfigure}[b]{0.49\textwidth}
        \centering
        \includegraphics[width=\textwidth]{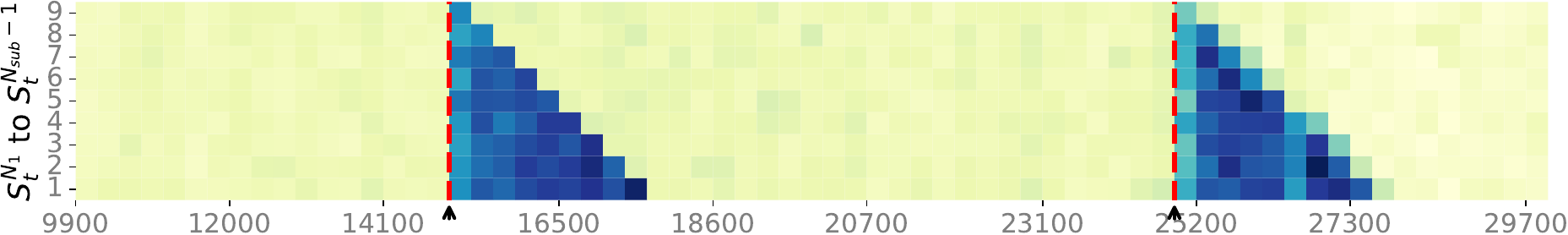}
        \vspace{-0.5cm}
        \caption{GM\_Rec.}
        \label{fig:hm_rec}
    \end{subfigure}
    \begin{subfigure}[b]{0.49\textwidth}
        \centering
        \includegraphics[width=\textwidth]{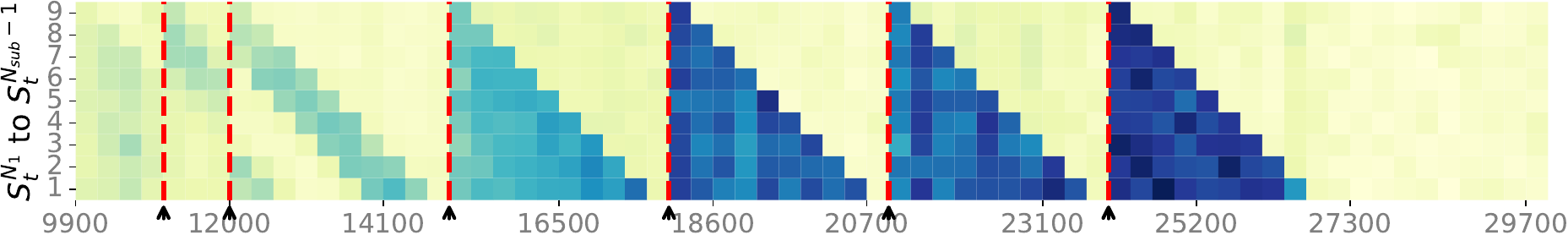}
        \vspace{-0.5cm}
        \caption{GM\_Grad.}
        \label{fig:hm_grad}
    \end{subfigure}
    \begin{subfigure}[b]{0.49\textwidth}
        \centering
        \includegraphics[width=\textwidth]{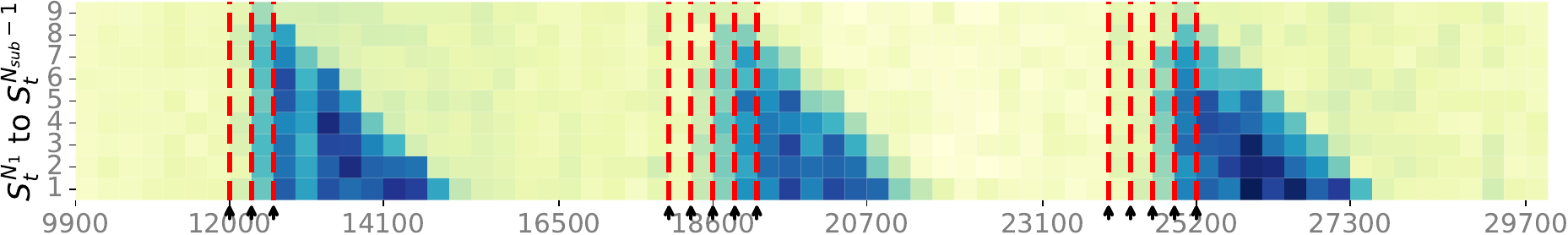}
        \vspace{-0.5cm}
        \caption{GM\_Inc.}
        \label{fig:hm_incre}
    \end{subfigure}
    \vspace{-0.4cm}
    \caption{Heatmaps of MCD between sub-windows for primary synthetic data sets with drift indicators (red lines and arrows).}
    \label{fig:hm}
    \vspace*{-0.3cm}
\end{figure*}

\subsection{Experiment Setting}
\subsubsection{\textbf{Data Sets}} \label{data sets}As summarized in Table \ref{tbl:datasets}, we used both \emph{synthetic} and \emph{real-world} data sets, each with known drift points.

The synthetic data sets are further divided into \emph{primary drift tasks} and \emph{complex drift tasks}, based on the complexity of drift detection. The primary drift tasks involve drifts created by altering the weights in Gaussian mixture models, leading to distinct distribution changes in low-dimensional data streams. Conversely, complex drift tasks include more challenging, higher-overlap distributions like Gamma, Lognormal, and Weibull, and consider multiple drift occurrences in higher-dimensional streams. A more comprehensive data generation process is provided in Appendix\,\ref{apdx:sdatasets}.

Concept drift in real-world data sets tends to occur with more arbitrary and complex causes and thus presents greater predictive challenges. We selected real data sets known for their drift locations and types. \textit{INSECTS}\,\cite{SouzaChallenges:2020} introduces data sets with multiple types of drift through varying environmental temperatures to collect sensor reading values monitoring mosquitos. \textit{EEG}\,\cite{misc_eeg_eye_state_264} is a collection of EEG measurements of eye states recorded via camera, where open and closed eye states represent two different EEG distributions causing concept drift. Both data sets are widely used in relevant work considering concept drifts in real-world scenarios\,\cite{Yoon_2022,10.1609/aaai.v33i01.33014594}. Further details on the real-world data sets can be found in Appendix\,\ref{apdx:rdatasets}.

For all data sets, a sliding window is employed from the beginning of each data set to simulate data streams. The sub-windows where the true drift initiates or is present are assigned a label of "Drift", while others are designated as "No Drift", formulating drift detection to a binary classification to facilitate evaluation.

\subsubsection{\textbf{Compared Algorithms}} 
We chose popular drift detection algorithms that can be adopted for \emph{unsupervised} and \emph{online} concept drift detection; (1) \emph{Kolmogorov-Smirnov Test (KS Test)}\,\cite{rabanser2019failing}: assessing whether one-dimensional distributions differ by calculating the supremum of the differences between their empirical distributions. For the multi-dimensional data, p-values were aggregated using the Bonferroni correction. (2) \emph{Maximum Mean Discrepancy with a Gaussian Kernel (MMD-GK)}\,\cite{JMLR:v13:gretton12a}: a kernel-based method for multi-dimensional two-sample testing that quantifies the disparity between two distributions by measuring their mean embeddings within the RKHS. In our experiments, we employ a Gaussian kernel to compute an unbiased estimate of the disparity. (3) \emph{Least-Squares Density Difference (LSDD)}\,\cite{7745962}: employing a linear-in-parameters Gaussian kernel function to estimate the distance between the probability density functions of two samples.  (4) \emph{Maximum Mean Discrepancy with a Deep Kernel (MMD-DK)}\,\cite{liu2020learning}: enhancing the traditional MMD approach by incorporating a deep kernel, where the kernel function is optimized using a subset of the data to maximize the test power. The detailed implementation and evaluation setting for the compared algorithms are provided in Appendix \ref{apdx:implementation_detail}

\subsubsection{\textbf{Evaluation Metrics}}

We used Precision, F1-score, and the Matthews Correlation Coefficient (MCC)\,\cite{chicco2020advantages} as metrics to gauge performance. Precision evaluates the proportion of true drifts correctly identified among all detected drifts, serving as a critical metric for assessing the effectiveness of a detector in avoiding false positives. The F1-score, a harmonic mean of precision and recall, serves as a balanced measure of a detector's effectiveness in identifying true drifts while considering both false positives and false negatives. The MCC\cite{chicco2020advantages}, which encompasses all quadrants of the confusion matrix, is particularly reliable in scenarios involving rare but critical events, notably in cases of concept drift. We conducted 20 runs for each experiment and reported the average with standard deviations.

\subsection{Drift Detection Accuracy} 
\label{sec:overall_performance}

\subsubsection{\textbf{Synthetic Data Sets}} 
As shown in Table \ref{tbl:overall_performance}, \algname{} achieved the highest precision across all simulated data sets, with the scores being almost all 1 or very close to 1, demonstrating its \textit{Eagle Eye} capability in drift point detection. Moreover, except for data sets with incremental drift, our method outperformed others in terms of F1-score and MCC as well.

\underline{Visualization of MCD in each window}: In Figure \ref{fig:hm}, We further analyzed the concept drift capability of \algname{}.
We calculated MCD values between the most recent data distributions in each new sub-window ($S_t^{N_{sub}}$) and all preceding data distributions within the same window ($S_t^{N_{1}}$ to $S_t^{N_{sub}-1}$). The horizontal axis represents the locations of $S_t^{N_{sub}}$ at each time step $t$, and the vertical axis encompasses preceding sub-windows in the same window. Then, each cell indicates the learned MCD between sub-window pairs. For various types of concept drift, the heatmap patterns exhibit distinctive shapes: sudden drift manifests as lower triangular patterns starting at the drift point, reoccurring drift as two lower triangles, gradual drift progressively forms lower triangular patterns from top to bottom, and incremental drift results in fainter lower triangles due to more subtle distributional changes. We further investigated the model's performance with even more subtle and slower incremental drifts and real-world data sets in Appendix \ref{apdx:additional_results}. 
These demonstrations highlight the interpretive strength of \algname{} in capturing the dynamics of drift occurrences.

\begin{figure}[t]
  \centering
  \includegraphics[width=\linewidth]{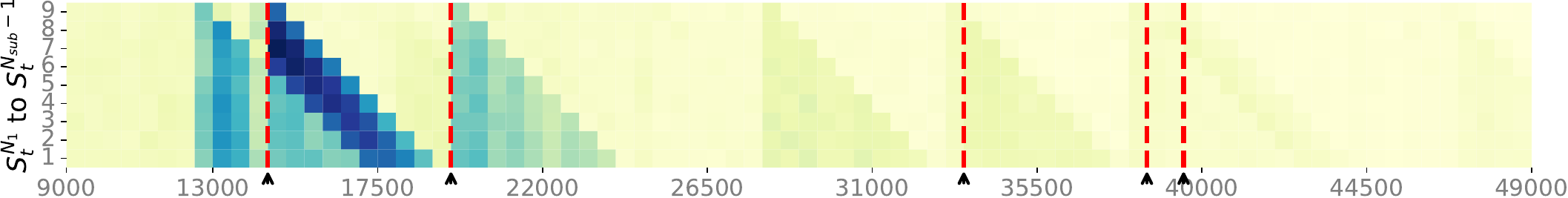}
  \vspace{-10pt}
    \vspace{-0.3cm}
  \caption{Heatmap for INSECTS\_Sud with drift indicators.}
  \label{fig:heatmap_insects_sud}
\vspace*{-0.4cm}
\end{figure}

\subsubsection{\textbf{Real-world Data Sets}}
As shown in Table \ref{tbl:overall_performance}, \algname{} demonstrated superior performances over other baseline algorithms across a variety of drift scenarios within the \textit{INSECTS}, including INSECTS\_Sud (sudden drift), INSECTS\_Grad (gradual drift), and INSECTS\_IncreRec (incremental drift and reoccurring). Figure \ref{fig:heatmap_insects_sud} shows the heatmap of MCD for INSECTS\_Sud. These results highlight \algname{}'s robust capability in effectively detecting and adapting to different types of drift phenomena, ranging from abrupt changes to slow evolutions and cyclic variations. 
Despite exhibiting slightly weaker performance than MMD-DK when applied to the \textit{EEG}, known for its high frequency of sudden drifts, it is noteworthy that \algname{} still consistently achieved the highest MCC value. 

\begin{figure}[t]
    \centering
    \includegraphics[width=\columnwidth]{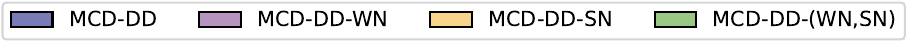}
    \includegraphics[width=0.8\columnwidth]{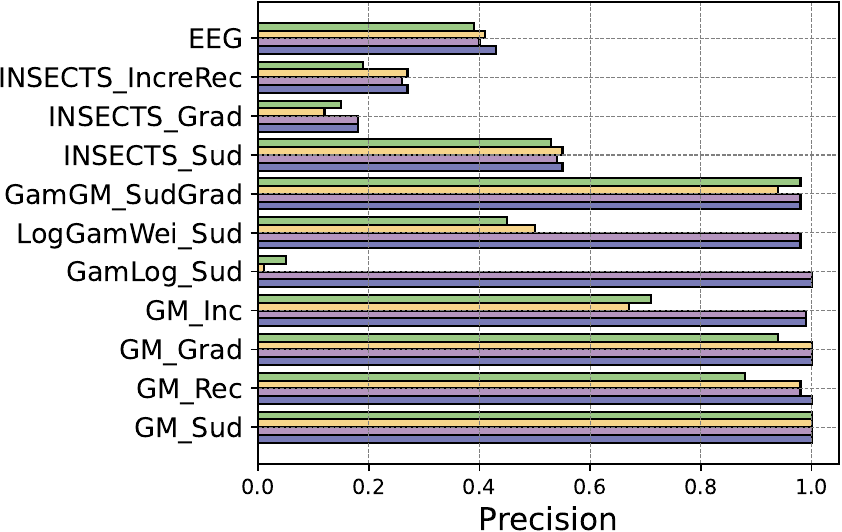}
    \vspace{-0.3cm}
    \caption{Ablation study of contrastive learning strategies.}
    \label{fig:ablation_pre}
    \vspace*{-0.4cm}
\end{figure}

\subsection{Ablation Study}
\label{sec:ablation}
We conducted ablation studies on the strategies for obtaining positive and negative sample pairs. Specifically, we evaluated the performance of \algname{} under various configurations: the complete \algname{} setup, \algname{} without weak negative sample pairs (i.e., \algname{}-WN), \algname{} without strong negative sample pairs (i.e., \algname{}-SN), and \algname{} utilizing only positive sample pairs, eliminating all negative samples (i.e., \algname{}-(WN,SN)).

\begin{figure}[!t]
    \centering
    \includegraphics[width=0.4\textwidth]{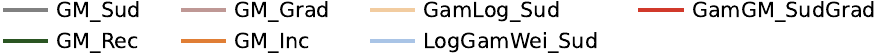}
    \begin{subfigure}[b]{0.22\textwidth}
        \centering
        \includegraphics[width=\textwidth]{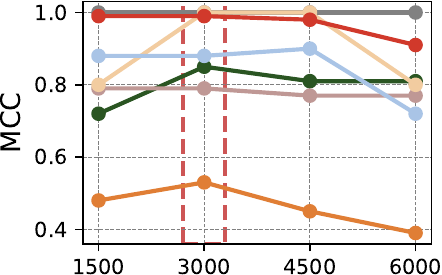}
        \caption{Window size $W$.}
        \vspace{0.15cm}
        \label{fig:sensitive_winsize}
    \end{subfigure}
    \begin{subfigure}[b]{0.22\textwidth}
        \centering
        \includegraphics[width=\textwidth]{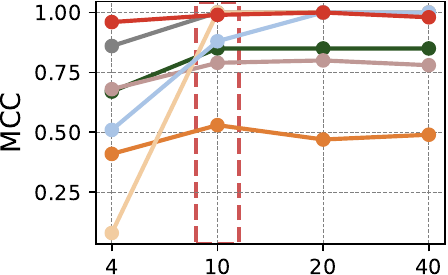}
        \caption{Number of sampling $k$.}
        \vspace{0.15cm}
        \label{fig:sensitive_k}
    \end{subfigure}
    \begin{subfigure}[b]{0.22\textwidth}
        \centering
        \includegraphics[width=\textwidth]{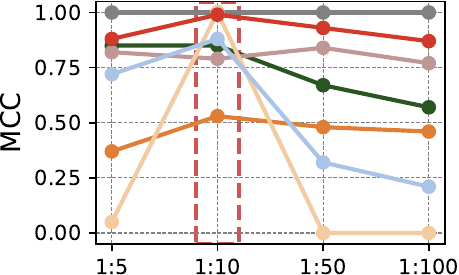}
        \caption{Ratio of noise $\epsilon_{\text{small}}:\epsilon_{\text{big}}$.}
        \vspace{0.15cm}
        \label{fig:sensitive_eps}
    \end{subfigure} 
    \begin{subfigure}[b]{0.22\textwidth}
        \centering
        \includegraphics[width=\textwidth]{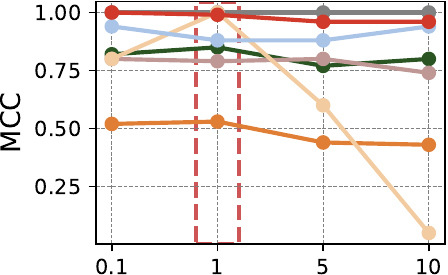}
        \caption{Regularization coefficient $\lambda$.}
        \vspace{0.15cm}
        \label{fig:sensitive_lamb}
    \end{subfigure}
    \vspace{-0.4cm}
    \caption{Sensitivity analysis results (default in red box).}
    \label{fig:sensitive}
    \vspace{-0.2cm}
\end{figure}

\begin{figure}[!t]
    \centering
    \includegraphics[width=0.4\textwidth]{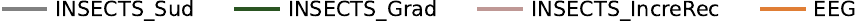}
    \begin{subfigure}[b]{0.22\textwidth}
        \centering
        \includegraphics[width=\textwidth]{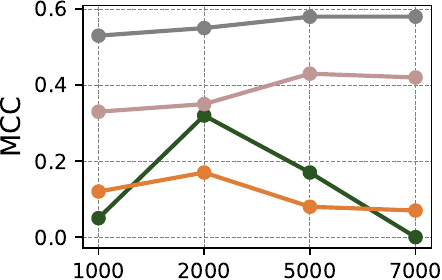}
        \caption{Window size $W$.}
        \vspace{0.15cm}
        \label{fig:sensitive_winsize_real}
    \end{subfigure}
    \begin{subfigure}[b]{0.22\textwidth}
        \centering
        \includegraphics[width=\textwidth]{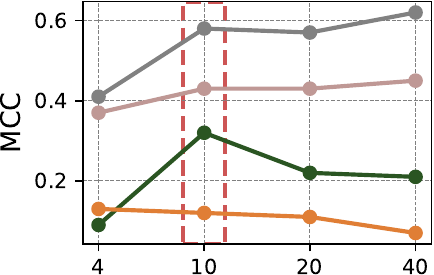}
        \caption{Number of sampling $k$.}
        \vspace{0.15cm}
        \label{fig:sensitive_k_real}
    \end{subfigure}
    \begin{subfigure}[b]{0.22\textwidth}
        \centering
        \includegraphics[width=\textwidth]{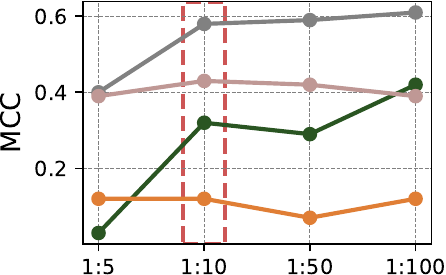}
        \caption{Ratio of noise $\epsilon_{\text{small}}:\epsilon_{\text{big}}$.}
        \vspace{0.15cm}
        \label{fig:sensitive_eps_real}
    \end{subfigure} 
    \begin{subfigure}[b]{0.22\textwidth}
        \centering
        \includegraphics[width=\textwidth]{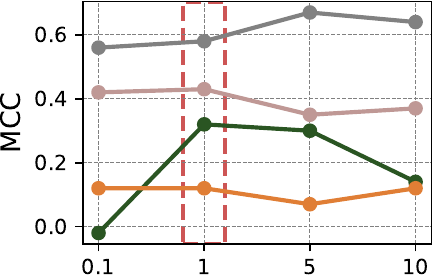}
        \caption{Regularization coefficient $\lambda$.}
        \vspace{0.15cm}
        \label{fig:sensitive_lamb_real}
    \end{subfigure}
    \vspace{-0.4cm}
    \caption{Sensitivity analysis results for real-world data sets.}
    \label{fig:sensitive_real}
    \vspace{-0.4cm}
\end{figure}

Figure \ref{fig:ablation_pre} shows the results of precision, while the results of other metrics showed similar trends (Appendix \ref{apdx:additional_results}). In summary, \algname{} achieved the highest precision across all data sets and the absence of each type of negative sample pairs leads to lower performances in most cases. Specifically, removing weak negative sample pairs resulted in diminished effectiveness in detecting gradual, subtle concept drifts, like GM\_Rec and EEG. Eliminating strong negative sample pairs adversely affected performance in more challenging drift detection scenarios (e.g., high overlap distributions, incremental drift), even failing to identify any drifts GamLog\_Sud, along with notably poor performances in GM\_Inc and INSECTS\_Grad. Retaining only positive sample pairs yielded comparable results in simpler tasks but significantly underperformed in more complex simulated and real-world data sets compared to strategies incorporating negative sample pairs. These ablation study findings demonstrate the efficacy of the sampling strategies proposed particularly in dealing with complex drift scenarios.

\begin{table}[t!]
\small
  \centering
  \caption{Window processing time (sec) over encoder sizes.}
  \vspace{-0.3cm}
  \begin{tabular}{@{}cccc@{}}
    \toprule
    Hidden Size & Sampling Time & Training Time & Inference Time \\
    \midrule
    50  & 0.0067 & 0.2106 & 0.0112 \\
    100 & 0.0065 & 0.2248 & 0.0117 \\
    150 & 0.0065 & 0.2297 & 0.0103 \\
    200 & 0.0069 & 0.2302 & 0.0120  \\
    250 & 0.0071 & 0.2456 & 0.0120  \\
    300 & 0.0067 & 0.2518 & 0.0121 \\
    \bottomrule
  \end{tabular}
  \label{tab:performance_time}
\vspace{-0.3cm}
\end{table}

\begin{figure}[t!]
    \centering
    \begin{subfigure}{\columnwidth}
        \centering
\includegraphics[width=\columnwidth, page=1]{Figures/sensitive_legend.pdf}
    \end{subfigure}
    \begin{subfigure}{\columnwidth}
        \centering
        \includegraphics[width=0.9\columnwidth, page=1]{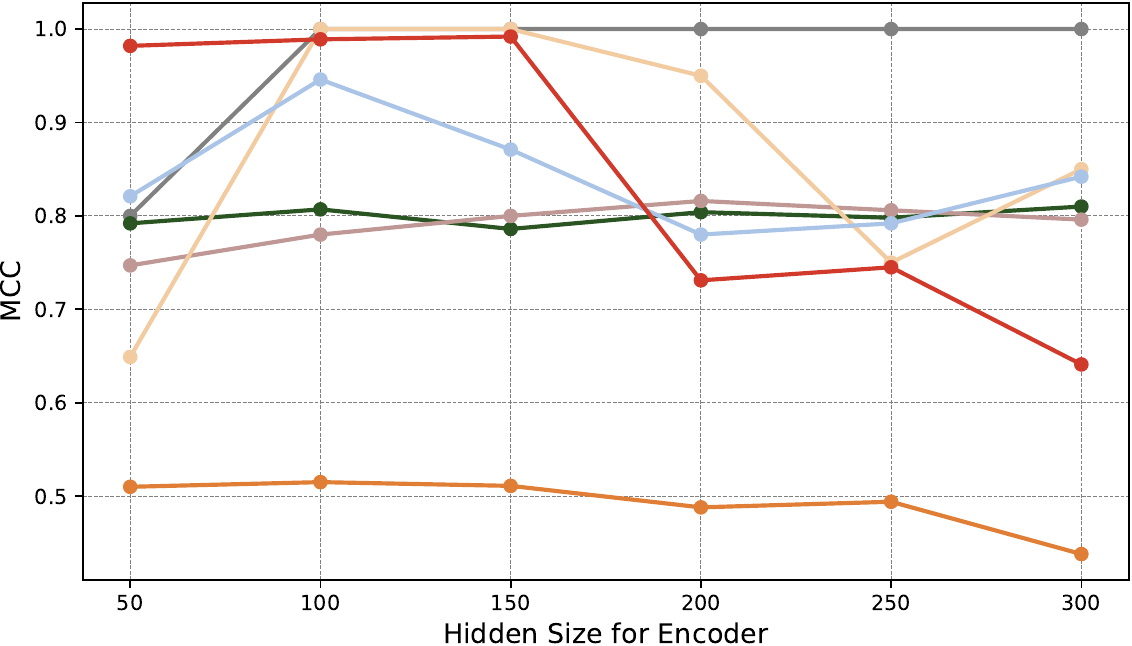}
    \end{subfigure}
    \vspace{-0.6cm}
    \caption{Detection accuracy over encoder sizes.}
    \label{fig:encoder1}
    \vspace{-0.7cm}
\end{figure}

\subsection{Sensitivity Analysis} 
\label{sec:sensitivity_analysis}
\subsubsection{\textbf{Effects of main Hyperparameters}} We conducted a sensitivity analysis of the main hyperparameters used in \algname{}: the \emph{sliding window size $W$}, the \emph{number of samples} $k$, the \emph{degree of noise} $\epsilon$, and the \emph{regularization coefficient} $\lambda$. Figure \ref{fig:sensitive} shows the MCC results on the synthetic data sets, and the results of other metrics showed similar trends. For $W$ we varied it from half to two times the default value. Figure \ref{fig:sensitive_winsize} demonstrates that \algname{} shows comparable performances over varying window sizes with the peak performances at window sizes around the default value. Since the window size determines the context of current concepts used for training the encoder, the smaller size is preferable in practice for efficiency, as long as it includes temporally distant different distributions. Regarding the number $k$ of sampling in each sub-window, Figure \ref{fig:sensitive_k} shows that the higher number of sampling leads to increased performance. Nevertheless, sampling frequencies beyond the default value ($k = 10$) lead to only marginal improvements. In constructing negative sample pairs, the perturbation degree $\delta$ with $\epsilon_{\text{small}}$ and $\epsilon_{\text{big}}$ controls the relative degrees of noise for weak and strong negative samples. Figure \ref{fig:sensitive_eps} indicates that too low or too high ratios of small and big noise perform poorly, particularly in challenging cases (e.g., GamLog\_Sud and LogGamWei\_Sud). The default ratios of $1:10$ were demonstrated to be the most optimal in most cases. Finally, regarding the regularization coefficient $\lambda$ for gradient penalty to ensure $L$-Lipschitz continuity in the optimization, Figure \ref{fig:sensitive_lamb} shows that either too light or too heavy penalties can reduce the model's effectiveness on challenging cases, making $\lambda=1$ a suitable choice.

Figure \ref{fig:sensitive_real} shows the sensitivity analysis results on real-world data sets, demonstrating similar trends with the synthetic data sets. While the default values for $W$ are not annotated in the figure given the varying lengths of each data set, it is found that setting $W$ to 10\% of the total length of each data set is a suitable choice.

\begin{figure*}[!t]
    \begin{subfigure}[b]{\textwidth}
        \centering
        \includegraphics[width=0.9\textwidth]{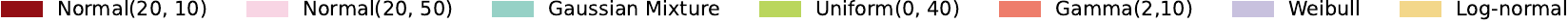}
        \vspace{0.2cm}
        
    \end{subfigure}
    
    \begin{subfigure}[b]{0.24\textwidth}
        \centering
        \includegraphics[width=\textwidth]{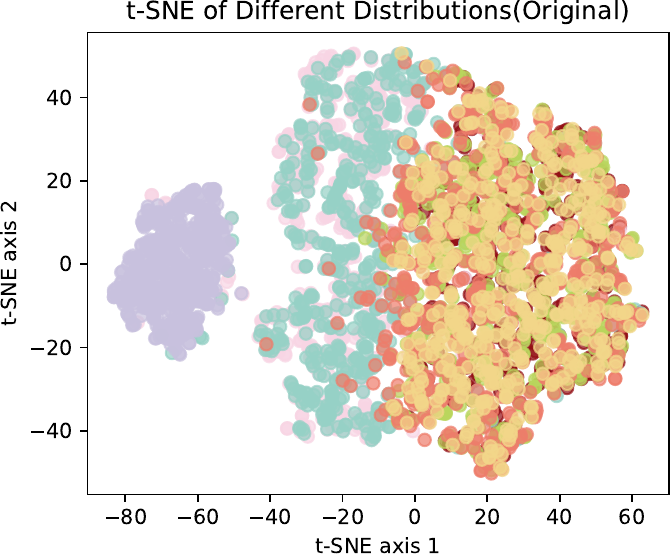}
        \vspace{-0.5cm}
        \caption{Orginal.}
        \label{fig:original}
    \end{subfigure}
    \begin{subfigure}[b]{0.24\textwidth}
        \centering
        \includegraphics[width=\textwidth]{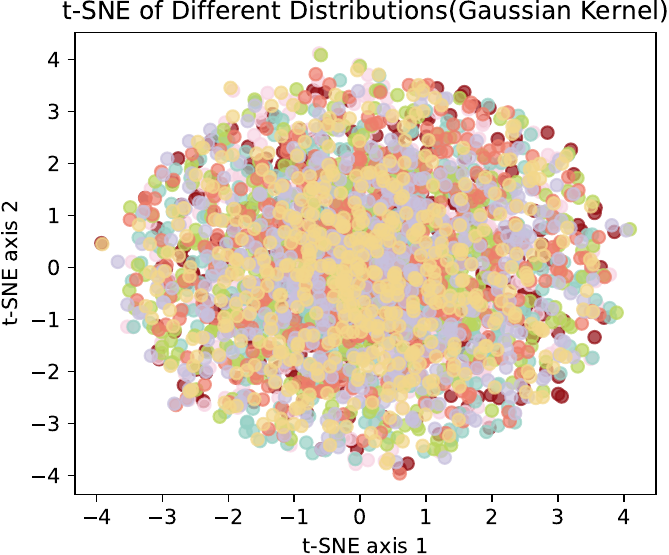}
        \vspace{-0.5cm}
        \caption{Gaussian Kernel.}
        \label{fig:gaussiankernel}
    \end{subfigure}
    \begin{subfigure}[b]{0.24\textwidth}
        \centering
        \includegraphics[width=\textwidth]{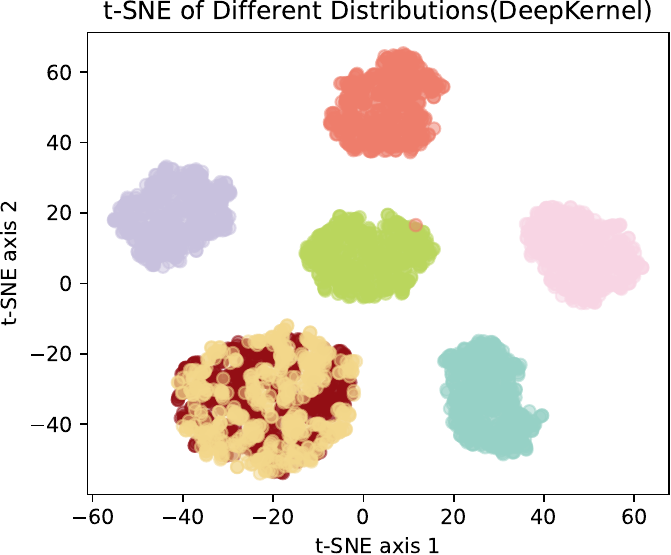}
        \vspace{-0.5cm}
        \caption{Deep Kernel.}
        \label{fig:deepkernel}
    \end{subfigure}
    \begin{subfigure}[b]{0.24\textwidth}
        \centering
        \includegraphics[width=\textwidth]{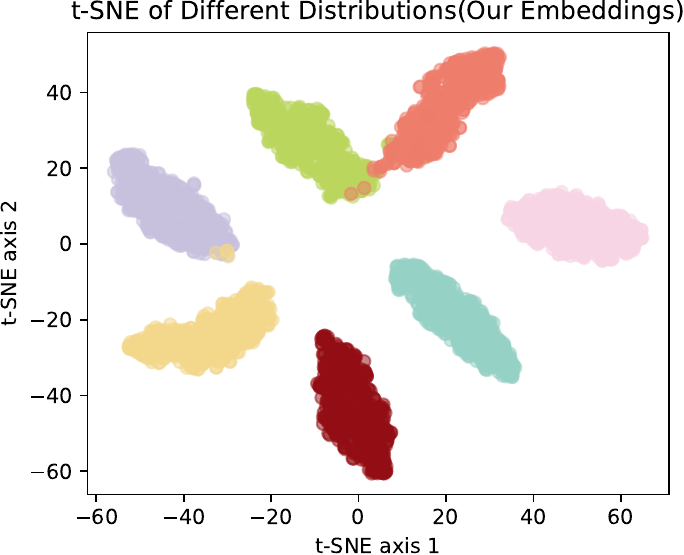}
        \vspace{-0.5cm}
        \caption{\algname{}.}
        \label{fig:ours}
    \end{subfigure}    
    \caption{Comparative visualization of discriminative embedding capabilities for complex distributions.}
    \label{fig:tsne}
\end{figure*}

\subsubsection{\textbf{Effects of Encoder Size}}
We also conducted a sensitivity analysis of the encoder on processing time (Table \ref{tab:performance_time}) and detection accuracy (Figure \ref{fig:encoder1}) by varying its hidden size from $50$ to $300$. The results indicate that, in general, the encoder's hidden size does not have a significant impact on performance. However, particularly in the data streams with complex drifts (e.g., GamLog\_Sud or GamGM\_SudGrad), the encoder with too low or too high hidden sizes degrades the detection accuracy owing to lacking expressiveness or overfitting in representing the complex concepts.

\subsection{Qualitative Analysis}
\label{sec:qualitative}
To more vividly showcase how \algname{} maps original data distributions to the appropriate embedding space, facilitating the identification of varying distribution changes, we visualized the embedding spaces generated from the compared algorithms in two-dimensional space by t-SNE\,\cite{van2008visualizing}.
The seven challenging distributions used to generate data points are detailed in Appendix \ref{apx:qualitative}. In Figure \ref{fig:tsne}, points in each t-SNE plot demonstrate the representations of a set of data, including 500 data points.
Different colors represent their different underlying probability distributions. 
Figure \ref{fig:original} demonstrates the difficulty of distinguishing distributions in the original space, showing highly intermingled clusters. 
Figure \ref{fig:gaussiankernel} shows that the Gaussian kernel mapping used by MMD-GK uniformly distributes all distributions across the space, aligning with the outcomes of prior studies\,\cite{liu2020learning}. While the deep kernel by MMD-DK achieved more separable embeddings, as shown in Figure \ref{fig:deepkernel}, \algname{} exhibits even better separation, particularly in handling more complex distributions. This aspect justifies the superior performance of \algname{} in concept drift detection across various scenarios.


\begin{figure}[t!]
    \centering
    \begin{subfigure}{\columnwidth}
        \centering
\includegraphics[width=\columnwidth]{Figures/sensitive_legend.pdf}
    \end{subfigure}
    \begin{subfigure}{\columnwidth}
        \centering
        \includegraphics[width=\columnwidth, page=1]{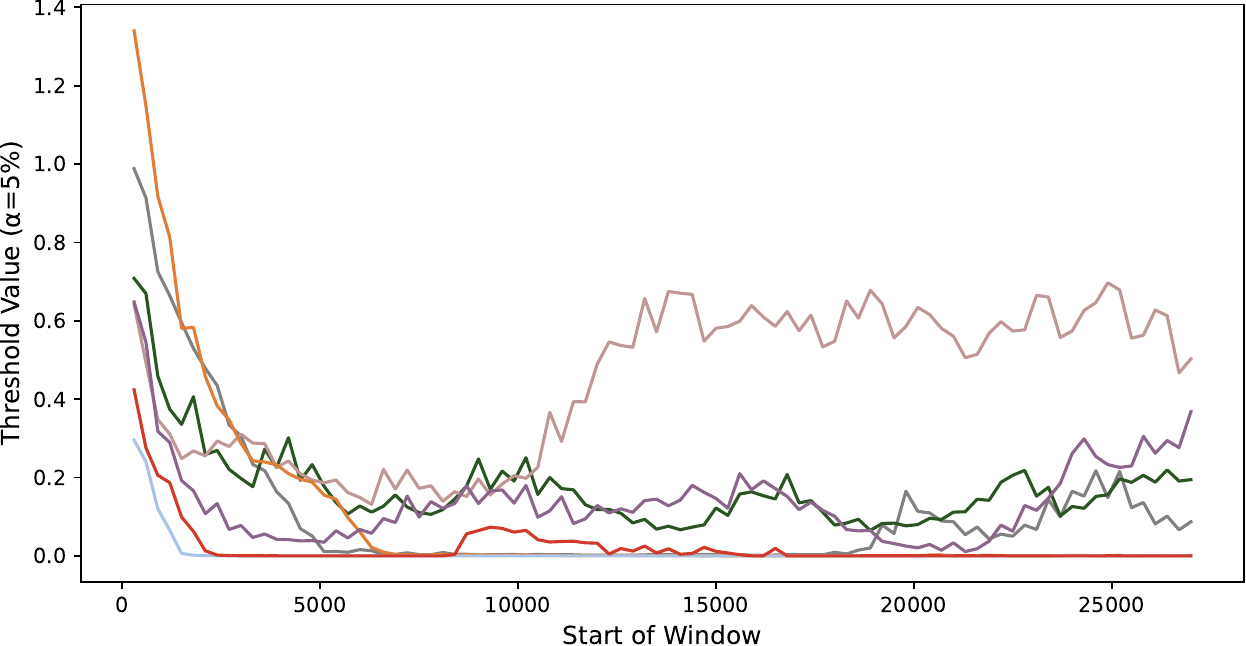}
    \end{subfigure}
    \caption{Changes of drift detection threshold \(\sigma\) over time.}
    \label{fig:threshold}
\end{figure}

\subsection{Drift Detection Threshold Analysis}
\label{sec:threshold}
As discussed in Section \ref{sec:drift_detector}, the drift detection threshold $\sigma$ of \algname{} is dynamically controlled by tracking the historical MCD values with a predefined statistical significance level (e.g., 0.05). To further understand its dynamicity, we analyzed the changes in the threshold over sliding windows in synthetic data sets with different drift types. Figure \ref{fig:threshold} shows that the thresholds were consistent over time in most cases, indicating that MCD is optimized robustly to varying drift types. GM\_Grad notably exhibited a significant increase in the threshold values after the first drift was observed around 10,000, but it quickly converged to a constant value and the corresponding detection accuracy remained high as shown in Section \ref{sec:overall_performance}.

\section{Conclusion}
We proposed \algname{}, an unsupervised online concept drift detection method, exploiting a new measure called maximum concept discrepancy. \algname{} leverages contrastive learning to obtain quality concept representations from sampled data points and optimize the maximum concept discrepancy. It is facilitated by sampling strategies based on temporal consistency and perturbations for robust optimization. The theoretical analysis demonstrated that \algname{} can also be used within the hypothesis testing framework. Experimental results on multiple synthetic and real-world data sets showed that the proposed method achieves superior detection accuracy and higher interpretability than existing baselines. 

For future work, we consider exploiting MCD histories learned throughout data streams. It can provide a systematic understanding of the patterns, duration, and strengths of concept drifts, as glimpsed in the heatmap visualization analyses. Further, assuming partial concept labels are available, adopting weak supervision philosophy can be promising. Estimating the degree of differences between partially labeled concepts can function as pseudo-labels to help us optimize the encoder more effectively.

\begin{acks}
This work was supported by ICT Creative Consilience Program through the Institute of Information \& Communications Technology Planning \& Evaluation(IITP) grant funded by the Korea government (MSIT) (IITP-2024-2020-0-01819) and the New Faculty Settlement Research Fund by Korea University. We thank Professor Yifan Cui from Zhejiang University for his valuable advice.

\end{acks}

\balance

\begin{thebibliography}{63}


\ifx \showCODEN    \undefined \def \showCODEN     #1{\unskip}     \fi
\ifx \showDOI      \undefined \def \showDOI       #1{#1}\fi
\ifx \showISBNx    \undefined \def \showISBNx     #1{\unskip}     \fi
\ifx \showISBNxiii \undefined \def \showISBNxiii  #1{\unskip}     \fi
\ifx \showISSN     \undefined \def \showISSN      #1{\unskip}     \fi
\ifx \showLCCN     \undefined \def \showLCCN      #1{\unskip}     \fi
\ifx \shownote     \undefined \def \shownote      #1{#1}          \fi
\ifx \showarticletitle \undefined \def \showarticletitle #1{#1}   \fi
\ifx \showURL      \undefined \def \showURL       {\relax}        \fi
\providecommand\bibfield[2]{#2}
\providecommand\bibinfo[2]{#2}
\providecommand\natexlab[1]{#1}
\providecommand\showeprint[2][]{arXiv:#2}

\bibitem[Agrahari and Singh(2022)]%
        {agrahari2022concept}
\bibfield{author}{\bibinfo{person}{Supriya Agrahari} {and} \bibinfo{person}{Anil~Kumar Singh}.} \bibinfo{year}{2022}\natexlab{}.
\newblock \showarticletitle{Concept drift detection in data stream mining: A literature review}.
\newblock \bibinfo{journal}{\emph{Journal of King Saud University-Computer and Information Sciences}} \bibinfo{volume}{34}, \bibinfo{number}{10} (\bibinfo{year}{2022}), \bibinfo{pages}{9523--9540}.
\newblock


\bibitem[Aguiar and Cano(2023)]%
        {aguiar2023comprehensive}
\bibfield{author}{\bibinfo{person}{Gabriel~J Aguiar} {and} \bibinfo{person}{Alberto Cano}.} \bibinfo{year}{2023}\natexlab{}.
\newblock \showarticletitle{A comprehensive analysis of concept drift locality in data streams}.
\newblock \bibinfo{journal}{\emph{arXiv preprint arXiv:2311.06396}} (\bibinfo{year}{2023}).
\newblock


\bibitem[Arjovsky et~al\mbox{.}(2017)]%
        {arjovsky2017wasserstein}
\bibfield{author}{\bibinfo{person}{Martin Arjovsky}, \bibinfo{person}{Soumith Chintala}, {and} \bibinfo{person}{L{\'e}on Bottou}.} \bibinfo{year}{2017}\natexlab{}.
\newblock \showarticletitle{Wasserstein generative adversarial networks}. In \bibinfo{booktitle}{\emph{ICML}}. PMLR, \bibinfo{pages}{214--223}.
\newblock


\bibitem[Bayram et~al\mbox{.}(2022)]%
        {bayram2022concept}
\bibfield{author}{\bibinfo{person}{Firas Bayram}, \bibinfo{person}{Bestoun~S Ahmed}, {and} \bibinfo{person}{Andreas Kassler}.} \bibinfo{year}{2022}\natexlab{}.
\newblock \showarticletitle{From concept drift to model degradation: An overview on performance-aware drift detectors}.
\newblock \bibinfo{journal}{\emph{Knowledge-Based Systems}}  \bibinfo{volume}{245} (\bibinfo{year}{2022}), \bibinfo{pages}{108632}.
\newblock


\bibitem[Berlinet and Thomas-Agnan(2011)]%
        {berlinet2011reproducing}
\bibfield{author}{\bibinfo{person}{Alain Berlinet} {and} \bibinfo{person}{Christine Thomas-Agnan}.} \bibinfo{year}{2011}\natexlab{}.
\newblock \bibinfo{booktitle}{\emph{Reproducing kernel Hilbert spaces in probability and statistics}}.
\newblock \bibinfo{publisher}{Springer Science \& Business Media}.
\newblock


\bibitem[Bifet and Gavalda(2007)]%
        {bifet2007learning}
\bibfield{author}{\bibinfo{person}{Albert Bifet} {and} \bibinfo{person}{Ricard Gavalda}.} \bibinfo{year}{2007}\natexlab{}.
\newblock \showarticletitle{Learning from time-changing data with adaptive windowing}. In \bibinfo{booktitle}{\emph{SDM}}. SIAM, \bibinfo{pages}{443--448}.
\newblock


\bibitem[Bifet et~al\mbox{.}(2023)]%
        {bifet2023machine}
\bibfield{author}{\bibinfo{person}{Albert Bifet}, \bibinfo{person}{Ricard Gavalda}, \bibinfo{person}{Geoffrey Holmes}, {and} \bibinfo{person}{Bernhard Pfahringer}.} \bibinfo{year}{2023}\natexlab{}.
\newblock \bibinfo{booktitle}{\emph{Machine learning for data streams: with practical examples in MOA}}.
\newblock \bibinfo{publisher}{MIT press}.
\newblock


\bibitem[Bu et~al\mbox{.}(2018)]%
        {7745962}
\bibfield{author}{\bibinfo{person}{Li Bu}, \bibinfo{person}{Cesare Alippi}, {and} \bibinfo{person}{Dongbin Zhao}.} \bibinfo{year}{2018}\natexlab{}.
\newblock \showarticletitle{A pdf-Free Change Detection Test Based on Density Difference Estimation}.
\newblock \bibinfo{journal}{\emph{IEEE TNNLS}} \bibinfo{volume}{29}, \bibinfo{number}{2} (\bibinfo{year}{2018}), \bibinfo{pages}{324--334}.
\newblock
\urldef\tempurl%
\url{https://doi.org/10.1109/TNNLS.2016.2619909}
\showDOI{\tempurl}


\bibitem[Cerqueira et~al\mbox{.}(2023)]%
        {cerqueira2023studd}
\bibfield{author}{\bibinfo{person}{Vitor Cerqueira}, \bibinfo{person}{Heitor~Murilo Gomes}, \bibinfo{person}{Albert Bifet}, {and} \bibinfo{person}{Luis Torgo}.} \bibinfo{year}{2023}\natexlab{}.
\newblock \showarticletitle{{STUDD}: A student--teacher method for unsupervised concept drift detection}.
\newblock \bibinfo{journal}{\emph{Machine Learning}} \bibinfo{volume}{112}, \bibinfo{number}{11} (\bibinfo{year}{2023}), \bibinfo{pages}{4351--4378}.
\newblock


\bibitem[Chen et~al\mbox{.}(2020)]%
        {chen2020simple}
\bibfield{author}{\bibinfo{person}{Ting Chen}, \bibinfo{person}{Simon Kornblith}, \bibinfo{person}{Mohammad Norouzi}, {and} \bibinfo{person}{Geoffrey Hinton}.} \bibinfo{year}{2020}\natexlab{}.
\newblock \showarticletitle{A simple framework for contrastive learning of visual representations}. In \bibinfo{booktitle}{\emph{ICML}}. PMLR, \bibinfo{pages}{1597--1607}.
\newblock


\bibitem[Chicco and Jurman(2020)]%
        {chicco2020advantages}
\bibfield{author}{\bibinfo{person}{Davide Chicco} {and} \bibinfo{person}{Giuseppe Jurman}.} \bibinfo{year}{2020}\natexlab{}.
\newblock \showarticletitle{The advantages of the Matthews correlation coefficient ({MCC}) over F1 score and accuracy in binary classification evaluation}.
\newblock \bibinfo{journal}{\emph{BMC genomics}} \bibinfo{volume}{21}, \bibinfo{number}{1} (\bibinfo{year}{2020}), \bibinfo{pages}{1--13}.
\newblock


\bibitem[Dasu et~al\mbox{.}(2006)]%
        {dasu2006information}
\bibfield{author}{\bibinfo{person}{Tamraparni Dasu}, \bibinfo{person}{Shankar Krishnan}, \bibinfo{person}{Suresh Venkatasubramanian}, {and} \bibinfo{person}{Ke Yi}.} \bibinfo{year}{2006}\natexlab{}.
\newblock \showarticletitle{An information-theoretic approach to detecting changes in multi-dimensional data streams}. In \bibinfo{booktitle}{\emph{Symposium on the Interface of Statistics, Computing Science, and Applications}}.
\newblock


\bibitem[Dziugaite et~al\mbox{.}(2015)]%
        {dziugaite2015training}
\bibfield{author}{\bibinfo{person}{Gintare~Karolina Dziugaite}, \bibinfo{person}{Daniel~M Roy}, {and} \bibinfo{person}{Zoubin Ghahramani}.} \bibinfo{year}{2015}\natexlab{}.
\newblock \showarticletitle{Training generative neural networks via maximum mean discrepancy optimization}. In \bibinfo{booktitle}{\emph{UAI}}. \bibinfo{pages}{258--267}.
\newblock


\bibitem[Elisseeff and Weston(2001)]%
        {elisseeff2001kernel}
\bibfield{author}{\bibinfo{person}{Andr{\'e} Elisseeff} {and} \bibinfo{person}{Jason Weston}.} \bibinfo{year}{2001}\natexlab{}.
\newblock \showarticletitle{A kernel method for multi-labelled classification}.
\newblock \bibinfo{journal}{\emph{NeurIPS}}  \bibinfo{volume}{14} (\bibinfo{year}{2001}).
\newblock


\bibitem[Frias-Blanco et~al\mbox{.}(2014)]%
        {frias2014online}
\bibfield{author}{\bibinfo{person}{Isvani Frias-Blanco}, \bibinfo{person}{Jos{\'e} del Campo-{\'A}vila}, \bibinfo{person}{Gonzalo Ramos-Jimenez}, \bibinfo{person}{Rafael Morales-Bueno}, \bibinfo{person}{Agust{\'\i}n Ortiz-D{\'\i}az}, {and} \bibinfo{person}{Yail{\'e} Caballero-Mota}.} \bibinfo{year}{2014}\natexlab{}.
\newblock \showarticletitle{Online and non-parametric drift detection methods based on Hoeffding’s bounds}.
\newblock \bibinfo{journal}{\emph{IEEE TKDE}} \bibinfo{volume}{27}, \bibinfo{number}{3} (\bibinfo{year}{2014}), \bibinfo{pages}{810--823}.
\newblock


\bibitem[Gama and Castillo(2006)]%
        {gama2006learning}
\bibfield{author}{\bibinfo{person}{Joao Gama} {and} \bibinfo{person}{Gladys Castillo}.} \bibinfo{year}{2006}\natexlab{}.
\newblock \showarticletitle{Learning with local drift detection}. In \bibinfo{booktitle}{\emph{ADMA}}. Springer, \bibinfo{pages}{42--55}.
\newblock


\bibitem[Gama et~al\mbox{.}(2004)]%
        {gama2004learning}
\bibfield{author}{\bibinfo{person}{Joao Gama}, \bibinfo{person}{Pedro Medas}, \bibinfo{person}{Gladys Castillo}, {and} \bibinfo{person}{Pedro Rodrigues}.} \bibinfo{year}{2004}\natexlab{}.
\newblock \showarticletitle{Learning with drift detection}. In \bibinfo{booktitle}{\emph{Advances in Artificial Intelligence--SBIA 2004}}. Springer, \bibinfo{pages}{286--295}.
\newblock


\bibitem[Gama et~al\mbox{.}(2013)]%
        {Gam13}
\bibfield{author}{\bibinfo{person}{Jo{\~a}o Gama}, \bibinfo{person}{Raquel Sebasti{\~a}o}, {and} \bibinfo{person}{Pedro~Pereira Rodrigues}.} \bibinfo{year}{2013}\natexlab{}.
\newblock \showarticletitle{On Evaluating Stream Learning Algorithms}.
\newblock \bibinfo{journal}{\emph{Machine Learning}} \bibinfo{volume}{90}, \bibinfo{number}{3} (\bibinfo{year}{2013}), \bibinfo{pages}{317--346}.
\newblock


\bibitem[Gama et~al\mbox{.}(2014)]%
        {gama2014survey}
\bibfield{author}{\bibinfo{person}{Jo{\~a}o Gama}, \bibinfo{person}{Indr{\.e} {\v{Z}}liobait{\.e}}, \bibinfo{person}{Albert Bifet}, \bibinfo{person}{Mykola Pechenizkiy}, {and} \bibinfo{person}{Abdelhamid Bouchachia}.} \bibinfo{year}{2014}\natexlab{}.
\newblock \showarticletitle{A survey on concept drift adaptation}.
\newblock \bibinfo{journal}{\emph{CSUR}} \bibinfo{volume}{46}, \bibinfo{number}{4} (\bibinfo{year}{2014}), \bibinfo{pages}{1--37}.
\newblock


\bibitem[Gemaque et~al\mbox{.}(2020)]%
        {gemaque2020overview}
\bibfield{author}{\bibinfo{person}{Rosana~Noronha Gemaque}, \bibinfo{person}{Albert Fran{\c{c}}a~Josu{\'a} Costa}, \bibinfo{person}{Rafael Giusti}, {and} \bibinfo{person}{Eulanda~Miranda Dos~Santos}.} \bibinfo{year}{2020}\natexlab{}.
\newblock \showarticletitle{An overview of unsupervised drift detection methods}.
\newblock \bibinfo{journal}{\emph{Wiley Interdisciplinary Reviews: Data Mining and Knowledge Discovery}} \bibinfo{volume}{10}, \bibinfo{number}{6} (\bibinfo{year}{2020}), \bibinfo{pages}{e1381}.
\newblock


\bibitem[Gouk et~al\mbox{.}(2021)]%
        {gouk2021regularisation}
\bibfield{author}{\bibinfo{person}{Henry Gouk}, \bibinfo{person}{Eibe Frank}, \bibinfo{person}{Bernhard Pfahringer}, {and} \bibinfo{person}{Michael~J Cree}.} \bibinfo{year}{2021}\natexlab{}.
\newblock \showarticletitle{Regularisation of neural networks by enforcing lipschitz continuity}.
\newblock \bibinfo{journal}{\emph{Machine Learning}}  \bibinfo{volume}{110} (\bibinfo{year}{2021}), \bibinfo{pages}{393--416}.
\newblock


\bibitem[Gretton et~al\mbox{.}(2006)]%
        {gretton2006kernel}
\bibfield{author}{\bibinfo{person}{Arthur Gretton}, \bibinfo{person}{Karsten Borgwardt}, \bibinfo{person}{Malte Rasch}, \bibinfo{person}{Bernhard Sch{\"o}lkopf}, {and} \bibinfo{person}{Alex Smola}.} \bibinfo{year}{2006}\natexlab{}.
\newblock \showarticletitle{A kernel method for the two-sample-problem}.
\newblock \bibinfo{journal}{\emph{NeurIPS}}  \bibinfo{volume}{19} (\bibinfo{year}{2006}).
\newblock


\bibitem[Gretton et~al\mbox{.}(2012)]%
        {JMLR:v13:gretton12a}
\bibfield{author}{\bibinfo{person}{Arthur Gretton}, \bibinfo{person}{Karsten~M. Borgwardt}, \bibinfo{person}{Malte~J. Rasch}, \bibinfo{person}{Bernhard Sch{{\"o}}lkopf}, {and} \bibinfo{person}{Alexander Smola}.} \bibinfo{year}{2012}\natexlab{}.
\newblock \showarticletitle{A Kernel Two-Sample Test}.
\newblock \bibinfo{journal}{\emph{Journal of Machine Learning Research}} \bibinfo{volume}{13}, \bibinfo{number}{25} (\bibinfo{year}{2012}), \bibinfo{pages}{723--773}.
\newblock
\urldef\tempurl%
\url{http://jmlr.org/papers/v13/gretton12a.html}
\showURL{%
\tempurl}


\bibitem[Gretton et~al\mbox{.}(2009)]%
        {gretton2009fast}
\bibfield{author}{\bibinfo{person}{Arthur Gretton}, \bibinfo{person}{Kenji Fukumizu}, \bibinfo{person}{Zaid Harchaoui}, {and} \bibinfo{person}{Bharath~K Sriperumbudur}.} \bibinfo{year}{2009}\natexlab{}.
\newblock \showarticletitle{A fast, consistent kernel two-sample test}.
\newblock \bibinfo{journal}{\emph{NeurIPS}}  \bibinfo{volume}{22} (\bibinfo{year}{2009}).
\newblock


\bibitem[Gulrajani et~al\mbox{.}(2017a)]%
        {gulrajani2017improved}
\bibfield{author}{\bibinfo{person}{Ishaan Gulrajani}, \bibinfo{person}{Faruk Ahmed}, \bibinfo{person}{Martin Arjovsky}, \bibinfo{person}{Vincent Dumoulin}, {and} \bibinfo{person}{Aaron~C Courville}.} \bibinfo{year}{2017}\natexlab{a}.
\newblock \showarticletitle{Improved training of wasserstein gans}.
\newblock \bibinfo{journal}{\emph{NeurIPS}}  \bibinfo{volume}{30} (\bibinfo{year}{2017}).
\newblock


\bibitem[Gulrajani et~al\mbox{.}(2017b)]%
        {eriksson2004lipschitz}
\bibfield{author}{\bibinfo{person}{Ishaan Gulrajani}, \bibinfo{person}{Faruk Ahmed}, \bibinfo{person}{Martin Arjovsky}, \bibinfo{person}{Vincent Dumoulin}, {and} \bibinfo{person}{Aaron~C Courville}.} \bibinfo{year}{2017}\natexlab{b}.
\newblock \showarticletitle{Improved training of wasserstein gans}.
\newblock \bibinfo{journal}{\emph{NeurIPS}}  \bibinfo{volume}{30} (\bibinfo{year}{2017}).
\newblock


\bibitem[He et~al\mbox{.}(2020)]%
        {he2020momentum}
\bibfield{author}{\bibinfo{person}{Kaiming He}, \bibinfo{person}{Haoqi Fan}, \bibinfo{person}{Yuxin Wu}, \bibinfo{person}{Saining Xie}, {and} \bibinfo{person}{Ross Girshick}.} \bibinfo{year}{2020}\natexlab{}.
\newblock \showarticletitle{Momentum contrast for unsupervised visual representation learning}. In \bibinfo{booktitle}{\emph{CVPR}}. \bibinfo{pages}{9729--9738}.
\newblock


\bibitem[Hofmann et~al\mbox{.}(2008)]%
        {hofmann2008kernel}
\bibfield{author}{\bibinfo{person}{Thomas Hofmann}, \bibinfo{person}{Bernhard Sch{\"o}lkopf}, {and} \bibinfo{person}{Alexander~J Smola}.} \bibinfo{year}{2008}\natexlab{}.
\newblock \showarticletitle{Kernel methods in machine learning}.
\newblock  (\bibinfo{year}{2008}).
\newblock


\bibitem[Kifer et~al\mbox{.}(2004)]%
        {kifer2004detecting}
\bibfield{author}{\bibinfo{person}{Daniel Kifer}, \bibinfo{person}{Shai Ben-David}, {and} \bibinfo{person}{Johannes Gehrke}.} \bibinfo{year}{2004}\natexlab{}.
\newblock \showarticletitle{Detecting change in data streams}. In \bibinfo{booktitle}{\emph{VLDB}}, Vol.~\bibinfo{volume}{4}. Toronto, Canada, \bibinfo{pages}{180--191}.
\newblock


\bibitem[Kim et~al\mbox{.}(2022)]%
        {kim2022covid}
\bibfield{author}{\bibinfo{person}{Doyoung Kim}, \bibinfo{person}{Hyangsuk Min}, \bibinfo{person}{Youngeun Nam}, \bibinfo{person}{Hwanjun Song}, \bibinfo{person}{Susik Yoon}, \bibinfo{person}{Minseok Kim}, {and} \bibinfo{person}{Jae-Gil Lee}.} \bibinfo{year}{2022}\natexlab{}.
\newblock \showarticletitle{Covid-{EEN}et: Predicting fine-grained impact of COVID-19 on local economies}. In \bibinfo{booktitle}{\emph{AAAI}}, Vol.~\bibinfo{volume}{36}. \bibinfo{pages}{11971--11981}.
\newblock


\bibitem[Korycki and Krawczyk(2021)]%
        {korycki2021concept}
\bibfield{author}{\bibinfo{person}{{\L}ukasz Korycki} {and} \bibinfo{person}{Bartosz Krawczyk}.} \bibinfo{year}{2021}\natexlab{}.
\newblock \showarticletitle{Concept drift detection from multi-class imbalanced data streams}. In \bibinfo{booktitle}{\emph{ICDE}}. IEEE, \bibinfo{pages}{1068--1079}.
\newblock


\bibitem[Li et~al\mbox{.}(2017)]%
        {li2017mmd}
\bibfield{author}{\bibinfo{person}{Chun-Liang Li}, \bibinfo{person}{Wei-Cheng Chang}, \bibinfo{person}{Yu Cheng}, \bibinfo{person}{Yiming Yang}, {and} \bibinfo{person}{Barnab{\'a}s P{\'o}czos}.} \bibinfo{year}{2017}\natexlab{}.
\newblock \showarticletitle{Mmd gan: Towards deeper understanding of moment matching network}.
\newblock \bibinfo{journal}{\emph{NeurIPS}}  \bibinfo{volume}{30} (\bibinfo{year}{2017}).
\newblock


\bibitem[Liu et~al\mbox{.}(2017)]%
        {liu2017fuzzy}
\bibfield{author}{\bibinfo{person}{Anjin Liu}, \bibinfo{person}{Guangquan Zhang}, {and} \bibinfo{person}{Jie Lu}.} \bibinfo{year}{2017}\natexlab{}.
\newblock \showarticletitle{Fuzzy time windowing for gradual concept drift adaptation}. In \bibinfo{booktitle}{\emph{FUZZ-IEEE}}. IEEE, \bibinfo{pages}{1--6}.
\newblock


\bibitem[Liu et~al\mbox{.}(2020)]%
        {liu2020learning}
\bibfield{author}{\bibinfo{person}{Feng Liu}, \bibinfo{person}{Wenkai Xu}, \bibinfo{person}{Jie Lu}, \bibinfo{person}{Guangquan Zhang}, \bibinfo{person}{Arthur Gretton}, {and} \bibinfo{person}{Danica~J. Sutherland}.} \bibinfo{year}{2020}\natexlab{}.
\newblock \showarticletitle{Learning Deep Kernels for Non-Parametric Two-Sample Tests}. In \bibinfo{booktitle}{\emph{ICML}}.
\newblock


\bibitem[Lu et~al\mbox{.}(2018)]%
        {lu2018learning}
\bibfield{author}{\bibinfo{person}{Jie Lu}, \bibinfo{person}{Anjin Liu}, \bibinfo{person}{Fan Dong}, \bibinfo{person}{Feng Gu}, \bibinfo{person}{Joao Gama}, {and} \bibinfo{person}{Guangquan Zhang}.} \bibinfo{year}{2018}\natexlab{}.
\newblock \showarticletitle{Learning under concept drift: A review}.
\newblock \bibinfo{journal}{\emph{IEEE TKDE}} \bibinfo{volume}{31}, \bibinfo{number}{12} (\bibinfo{year}{2018}), \bibinfo{pages}{2346--2363}.
\newblock


\bibitem[Miyaguchi and Kajino(2019)]%
        {10.1609/aaai.v33i01.33014594}
\bibfield{author}{\bibinfo{person}{Kohei Miyaguchi} {and} \bibinfo{person}{Hiroshi Kajino}.} \bibinfo{year}{2019}\natexlab{}.
\newblock \showarticletitle{Cogra: concept-drift-aware stochastic gradient descent for time-series forecasting}. In \bibinfo{booktitle}{\emph{AAAI}} (Honolulu, Hawaii, USA). Article \bibinfo{articleno}{564}, \bibinfo{numpages}{8}~pages.
\newblock
\showISBNx{978-1-57735-809-1}
\urldef\tempurl%
\url{https://doi.org/10.1609/aaai.v33i01.33014594}
\showDOI{\tempurl}


\bibitem[Oord et~al\mbox{.}(2018)]%
        {oord2018representation}
\bibfield{author}{\bibinfo{person}{Aaron van~den Oord}, \bibinfo{person}{Yazhe Li}, {and} \bibinfo{person}{Oriol Vinyals}.} \bibinfo{year}{2018}\natexlab{}.
\newblock \showarticletitle{Representation learning with contrastive predictive coding}.
\newblock \bibinfo{journal}{\emph{arXiv preprint arXiv:1807.03748}} (\bibinfo{year}{2018}).
\newblock


\bibitem[Raab et~al\mbox{.}(2020)]%
        {raab2020reactive}
\bibfield{author}{\bibinfo{person}{Christoph Raab}, \bibinfo{person}{Moritz Heusinger}, {and} \bibinfo{person}{Frank-Michael Schleif}.} \bibinfo{year}{2020}\natexlab{}.
\newblock \showarticletitle{Reactive soft prototype computing for concept drift streams}.
\newblock \bibinfo{journal}{\emph{Neurocomputing}}  \bibinfo{volume}{416} (\bibinfo{year}{2020}), \bibinfo{pages}{340--351}.
\newblock


\bibitem[Rabanser et~al\mbox{.}(2019)]%
        {rabanser2019failing}
\bibfield{author}{\bibinfo{person}{Stephan Rabanser}, \bibinfo{person}{Stephan G{\"u}nnemann}, {and} \bibinfo{person}{Zachary Lipton}.} \bibinfo{year}{2019}\natexlab{}.
\newblock \showarticletitle{Failing loudly: An empirical study of methods for detecting dataset shift}.
\newblock \bibinfo{journal}{\emph{NeurIPS}}  \bibinfo{volume}{32} (\bibinfo{year}{2019}).
\newblock


\bibitem[Rahimi and Recht(2007)]%
        {NIPS2007_013a006f}
\bibfield{author}{\bibinfo{person}{Ali Rahimi} {and} \bibinfo{person}{Benjamin Recht}.} \bibinfo{year}{2007}\natexlab{}.
\newblock \showarticletitle{Random Features for Large-Scale Kernel Machines}. In \bibinfo{booktitle}{\emph{NeurIPS}}, \bibfield{editor}{\bibinfo{person}{J.~Platt}, \bibinfo{person}{D.~Koller}, \bibinfo{person}{Y.~Singer}, {and} \bibinfo{person}{S.~Roweis}} (Eds.), Vol.~\bibinfo{volume}{20}. \bibinfo{publisher}{Curran Associates, Inc.}
\newblock
\urldef\tempurl%
\url{https://proceedings.neurips.cc/paper_files/paper/2007/file/013a006f03dbc5392effeb8f18fda755-Paper.pdf}
\showURL{%
\tempurl}


\bibitem[Roesler(2013)]%
        {misc_eeg_eye_state_264}
\bibfield{author}{\bibinfo{person}{Oliver Roesler}.} \bibinfo{year}{2013}\natexlab{}.
\newblock \bibinfo{title}{{EEG Eye State}}.
\newblock \bibinfo{howpublished}{UCI Machine Learning Repository}.
\newblock
\newblock
\shownote{{DOI}: https://doi.org/10.24432/C57G7J}.


\bibitem[Shao et~al\mbox{.}(2014)]%
        {shao2014prototype}
\bibfield{author}{\bibinfo{person}{Junming Shao}, \bibinfo{person}{Zahra Ahmadi}, {and} \bibinfo{person}{Stefan Kramer}.} \bibinfo{year}{2014}\natexlab{}.
\newblock \showarticletitle{Prototype-based learning on concept-drifting data streams}. In \bibinfo{booktitle}{\emph{KDD}}. \bibinfo{pages}{412--421}.
\newblock


\bibitem[Shin et~al\mbox{.}(2021)]%
        {shin2021coherence}
\bibfield{author}{\bibinfo{person}{Yooju Shin}, \bibinfo{person}{Susik Yoon}, \bibinfo{person}{Sundong Kim}, \bibinfo{person}{Hwanjun Song}, \bibinfo{person}{Jae-Gil Lee}, {and} \bibinfo{person}{Byung~Suk Lee}.} \bibinfo{year}{2021}\natexlab{}.
\newblock \showarticletitle{Coherence-based label propagation over time series for accelerated active learning}. In \bibinfo{booktitle}{\emph{ICLR}}.
\newblock


\bibitem[Smola et~al\mbox{.}(2006)]%
        {smola2006maximum}
\bibfield{author}{\bibinfo{person}{Alexander~J Smola}, \bibinfo{person}{A Gretton}, {and} \bibinfo{person}{K Borgwardt}.} \bibinfo{year}{2006}\natexlab{}.
\newblock \showarticletitle{Maximum mean discrepancy}. In \bibinfo{booktitle}{\emph{ICONIP}}. \bibinfo{pages}{3--6}.
\newblock


\bibitem[Song et~al\mbox{.}(2007)]%
        {song2007statistical}
\bibfield{author}{\bibinfo{person}{Xiuyao Song}, \bibinfo{person}{Mingxi Wu}, \bibinfo{person}{Christopher Jermaine}, {and} \bibinfo{person}{Sanjay Ranka}.} \bibinfo{year}{2007}\natexlab{}.
\newblock \showarticletitle{Statistical change detection for multi-dimensional data}. In \bibinfo{booktitle}{\emph{KDD}}. \bibinfo{pages}{667--676}.
\newblock


\bibitem[Souza et~al\mbox{.}(2020)]%
        {SouzaChallenges:2020}
\bibfield{author}{\bibinfo{person}{V.~M.~A. Souza}, \bibinfo{person}{D.~M. Reis}, \bibinfo{person}{A.~G. Maletzke}, {and} \bibinfo{person}{G.~E. A. P.~A. Batista}.} \bibinfo{year}{2020}\natexlab{}.
\newblock \showarticletitle{Challenges in Benchmarking Stream Learning Algorithms with Real-world Data}.
\newblock \bibinfo{journal}{\emph{Data Mining and Knowledge Discovery}}  \bibinfo{volume}{34} (\bibinfo{year}{2020}), \bibinfo{pages}{1805--1858}.
\newblock
\urldef\tempurl%
\url{https://doi.org/10.1007/s10618-020-00698-5}
\showDOI{\tempurl}


\bibitem[Trirat et~al\mbox{.}(2023)]%
        {trirat2023mg}
\bibfield{author}{\bibinfo{person}{Patara Trirat}, \bibinfo{person}{Susik Yoon}, {and} \bibinfo{person}{Jae-Gil Lee}.} \bibinfo{year}{2023}\natexlab{}.
\newblock \showarticletitle{{MG-TAR}: multi-view graph convolutional networks for traffic accident risk prediction}.
\newblock \bibinfo{journal}{\emph{IEEE Transactions on Intelligent Transportation Systems}} \bibinfo{volume}{24}, \bibinfo{number}{4} (\bibinfo{year}{2023}), \bibinfo{pages}{3779--3794}.
\newblock


\bibitem[Van~der Maaten and Hinton(2008)]%
        {van2008visualizing}
\bibfield{author}{\bibinfo{person}{Laurens Van~der Maaten} {and} \bibinfo{person}{Geoffrey Hinton}.} \bibinfo{year}{2008}\natexlab{}.
\newblock \showarticletitle{Visualizing data using t-SNE.}
\newblock \bibinfo{journal}{\emph{Journal of machine learning research}} \bibinfo{volume}{9}, \bibinfo{number}{11} (\bibinfo{year}{2008}).
\newblock


\bibitem[Van~Looveren et~al\mbox{.}(2019)]%
        {alibi-detect}
\bibfield{author}{\bibinfo{person}{Arnaud Van~Looveren}, \bibinfo{person}{Janis Klaise}, \bibinfo{person}{Giovanni Vacanti}, \bibinfo{person}{Oliver Cobb}, \bibinfo{person}{Ashley Scillitoe}, \bibinfo{person}{Robert Samoilescu}, {and} \bibinfo{person}{Alex Athorne}.} \bibinfo{year}{2019}\natexlab{}.
\newblock \bibinfo{booktitle}{\emph{Alibi Detect: Algorithms for outlier, adversarial and drift detection}}.
\newblock
\urldef\tempurl%
\url{https://github.com/SeldonIO/alibi-detect}
\showURL{%
\tempurl}


\bibitem[Vedaldi and Zisserman(2012)]%
        {6136519}
\bibfield{author}{\bibinfo{person}{Andrea Vedaldi} {and} \bibinfo{person}{Andrew Zisserman}.} \bibinfo{year}{2012}\natexlab{}.
\newblock \showarticletitle{Efficient Additive Kernels via Explicit Feature Maps}.
\newblock \bibinfo{journal}{\emph{IEEE TPAMI}} \bibinfo{volume}{34}, \bibinfo{number}{3} (\bibinfo{year}{2012}), \bibinfo{pages}{480--492}.
\newblock
\urldef\tempurl%
\url{https://doi.org/10.1109/TPAMI.2011.153}
\showDOI{\tempurl}


\bibitem[Wang and Metze(2016)]%
        {10.1145/2911996.2912048}
\bibfield{author}{\bibinfo{person}{Yun Wang} {and} \bibinfo{person}{Florian Metze}.} \bibinfo{year}{2016}\natexlab{}.
\newblock \showarticletitle{Recurrent Support Vector Machines for Audio-Based Multimedia Event Detection}. In \bibinfo{booktitle}{\emph{ICMR}} (New York, New York, USA) \emph{(\bibinfo{series}{ICMR '16})}. \bibinfo{publisher}{Association for Computing Machinery}, \bibinfo{address}{New York, NY, USA}, \bibinfo{pages}{265–269}.
\newblock
\showISBNx{9781450343596}
\urldef\tempurl%
\url{https://doi.org/10.1145/2911996.2912048}
\showDOI{\tempurl}


\bibitem[Wang et~al\mbox{.}(2021)]%
        {wang2021clear}
\bibfield{author}{\bibinfo{person}{Zhuoyi Wang}, \bibinfo{person}{Yuqiao Chen}, \bibinfo{person}{Chen Zhao}, \bibinfo{person}{Yu Lin}, \bibinfo{person}{Xujiang Zhao}, \bibinfo{person}{Hemeng Tao}, \bibinfo{person}{Yigong Wang}, {and} \bibinfo{person}{Latifur Khan}.} \bibinfo{year}{2021}\natexlab{}.
\newblock \showarticletitle{Clear: Contrastive-prototype learning with drift estimation for resource constrained stream mining}. In \bibinfo{booktitle}{\emph{TheWebConf}}. \bibinfo{pages}{1351--1362}.
\newblock


\bibitem[Wang et~al\mbox{.}(2020)]%
        {wang2020boundary}
\bibfield{author}{\bibinfo{person}{Zhenzhi Wang}, \bibinfo{person}{Ziteng Gao}, \bibinfo{person}{Limin Wang}, \bibinfo{person}{Zhifeng Li}, {and} \bibinfo{person}{Gangshan Wu}.} \bibinfo{year}{2020}\natexlab{}.
\newblock \showarticletitle{Boundary-aware cascade networks for temporal action segmentation}. In \bibinfo{booktitle}{\emph{ECCV}}. Springer, \bibinfo{pages}{34--51}.
\newblock


\bibitem[Wang et~al\mbox{.}(2022)]%
        {wang2022online}
\bibfield{author}{\bibinfo{person}{Zhen Wang}, \bibinfo{person}{Liu Liu}, \bibinfo{person}{Yajing Kong}, \bibinfo{person}{Jiaxian Guo}, {and} \bibinfo{person}{Dacheng Tao}.} \bibinfo{year}{2022}\natexlab{}.
\newblock \showarticletitle{Online continual learning with contrastive vision transformer}. In \bibinfo{booktitle}{\emph{ECCV}}. Springer, \bibinfo{pages}{631--650}.
\newblock


\bibitem[Wilson et~al\mbox{.}(2016)]%
        {wilson2016deep}
\bibfield{author}{\bibinfo{person}{Andrew~Gordon Wilson}, \bibinfo{person}{Zhiting Hu}, \bibinfo{person}{Ruslan Salakhutdinov}, {and} \bibinfo{person}{Eric~P Xing}.} \bibinfo{year}{2016}\natexlab{}.
\newblock \showarticletitle{Deep kernel learning}. In \bibinfo{booktitle}{\emph{AISTATS}}. PMLR, \bibinfo{pages}{370--378}.
\newblock


\bibitem[Xu and Wang(2017)]%
        {xu2017dynamic}
\bibfield{author}{\bibinfo{person}{Shuliang Xu} {and} \bibinfo{person}{Junhong Wang}.} \bibinfo{year}{2017}\natexlab{}.
\newblock \showarticletitle{Dynamic extreme learning machine for data stream classification}.
\newblock \bibinfo{journal}{\emph{Neurocomputing}}  \bibinfo{volume}{238} (\bibinfo{year}{2017}), \bibinfo{pages}{433--449}.
\newblock


\bibitem[Yoon et~al\mbox{.}(2023a)]%
        {pdsum}
\bibfield{author}{\bibinfo{person}{Susik Yoon}, \bibinfo{person}{Hou~Pong Chan}, {and} \bibinfo{person}{Jiawei Han}.} \bibinfo{year}{2023}\natexlab{a}.
\newblock \showarticletitle{{PDSum}: Prototype-driven continuous summarization of evolving multi-document sets stream}. In \bibinfo{booktitle}{\emph{Proceedings of the ACM Web Conference 2023}}. \bibinfo{pages}{1650--1661}.
\newblock


\bibitem[Yoon et~al\mbox{.}(2023b)]%
        {ustory}
\bibfield{author}{\bibinfo{person}{Susik Yoon}, \bibinfo{person}{Dongha Lee}, \bibinfo{person}{Yunyi Zhang}, {and} \bibinfo{person}{Jiawei Han}.} \bibinfo{year}{2023}\natexlab{b}.
\newblock \showarticletitle{Unsupervised story discovery from continuous news streams via scalable thematic embedding}. In \bibinfo{booktitle}{\emph{Proceedings of the 46th International ACM SIGIR Conference on Research and Development in Information Retrieval}}. \bibinfo{pages}{802--811}.
\newblock


\bibitem[Yoon et~al\mbox{.}(2019)]%
        {nets}
\bibfield{author}{\bibinfo{person}{Susik Yoon}, \bibinfo{person}{Jae-Gil Lee}, {and} \bibinfo{person}{Byung~Suk Lee}.} \bibinfo{year}{2019}\natexlab{}.
\newblock \showarticletitle{{NETS}: extremely fast outlier detection from a data stream via set-based processing}.
\newblock \bibinfo{journal}{\emph{Proceedings of the VLDB Endowment}} \bibinfo{volume}{12}, \bibinfo{number}{11} (\bibinfo{year}{2019}), \bibinfo{pages}{1303--1315}.
\newblock


\bibitem[Yoon et~al\mbox{.}(2022)]%
        {Yoon_2022}
\bibfield{author}{\bibinfo{person}{Susik Yoon}, \bibinfo{person}{Youngjun Lee}, \bibinfo{person}{Jae-Gil Lee}, {and} \bibinfo{person}{Byung~Suk Lee}.} \bibinfo{year}{2022}\natexlab{}.
\newblock \showarticletitle{Adaptive Model Pooling for Online Deep Anomaly Detection from a Complex Evolving Data Stream}. In \bibinfo{booktitle}{\emph{KDD}}. \bibinfo{publisher}{ACM}.
\newblock
\urldef\tempurl%
\url{https://doi.org/10.1145/3534678.3539348}
\showDOI{\tempurl}


\bibitem[Yoon et~al\mbox{.}(2023c)]%
        {yoon2023scstory}
\bibfield{author}{\bibinfo{person}{Susik Yoon}, \bibinfo{person}{Yu Meng}, \bibinfo{person}{Dongha Lee}, {and} \bibinfo{person}{Jiawei Han}.} \bibinfo{year}{2023}\natexlab{c}.
\newblock \showarticletitle{{SCS}tory: Self-supervised and Continual Online Story Discovery}. In \bibinfo{booktitle}{\emph{TheWebConf}}. \bibinfo{pages}{1853--1864}.
\newblock


\bibitem[Yu et~al\mbox{.}(2023)]%
        {yu2023type}
\bibfield{author}{\bibinfo{person}{Hang Yu}, \bibinfo{person}{Jinpeng Li}, \bibinfo{person}{Jie Lu}, \bibinfo{person}{Yiliao Song}, \bibinfo{person}{Shaorong Xie}, {and} \bibinfo{person}{Guangquan Zhang}.} \bibinfo{year}{2023}\natexlab{}.
\newblock \showarticletitle{Type-{LDD}: A Type-Driven Lite Concept Drift Detector for Data Streams}.
\newblock \bibinfo{journal}{\emph{IEEE TKDE}} (\bibinfo{year}{2023}).
\newblock


\bibitem[Yu et~al\mbox{.}(2021)]%
        {yu2021automatic}
\bibfield{author}{\bibinfo{person}{Hang Yu}, \bibinfo{person}{Tianyu Liu}, \bibinfo{person}{Jie Lu}, {and} \bibinfo{person}{Guangquan Zhang}.} \bibinfo{year}{2021}\natexlab{}.
\newblock \showarticletitle{Automatic learning to detect concept drift}.
\newblock \bibinfo{journal}{\emph{arXiv preprint arXiv:2105.01419}} (\bibinfo{year}{2021}).
\newblock


\end{thebibliography}


\bibliographystyle{abbrv}

\clearpage
\appendix

\section{Appendix}

\subsection{Data Sets}
\label{apdx:datasets}

\subsubsection{\textbf{Synthetic Data Sets}}
\label{apdx:sdatasets}
The synthetic data sets used for model performance evaluation feature two types of drifts:\textbf{ primary }and \textbf{complex}. Primary drift tasks involve concept drifts that are relatively easy to distinguish within the data stream. Complex drift tasks, however, present higher distribution similarity, increased dimensionality, and a greater number of drift events, making them more challenging to discern.

For primary tasks, drift is simulated by adjusting the weighting of two Gaussian distributions, \( G_1 \) and \( G_2 \). Both distributions conform to a 5-dimensional Gaussian distribution with a mean vector \( \mu = [20, 20, 20, 20, 20] \) and covariance matrices \( \Sigma_1 = 10^2\mathbf{I} \) and \( \Sigma_2 = 50^2\mathbf{I} \), where \( \mathbf{I} \) denotes the identity matrix. These distributions form the basis for diverse Gaussian mixtures, with the weighting factor \( p \) controlling the proportion of each distribution in the mixture. The mixture model is represented by the equation: \(\mathcal{N}(\mu, \Sigma_1) \times p + \mathcal{N}(\mu, \Sigma_2) \times (1 - p)\). Each primary drift task involves a data set of 30,000 instances. Within this framework, different types of primary drift tasks are simulated by varying \( p \):

\textbf{\underline{Primary Task 1-GM\_Sud}}: Initially, \( p \) is set to 0.2. To induce a sudden drift at the 21,000th instance, \( p \) is shifted to 0.8, significantly changing the mixture's composition and simulating a sudden change in the underlying data structure.

\textbf{\underline{Primary Task 2-GM\_Rec}}: Initially, \( p \) is set to 0.8, giving prominence to \( G_1 \). At the 15,000th instance, \( p \) changes to 0.2, thus shifting the mixture's balance towards \( G_2 \). Finally, at the 25,000th instance, \( p \) returns to 0.8, reinstating the initial distribution emphasis. This pattern creates a reoccurring drift in the data stream.

\textbf{\underline{Primary Task 3-GM\_Inc}}: Initially set at 0.2, the weighting factor \(p\) undergoes specific adjustments to introduce incremental drifts at predetermined intervals within the data set. Specifically, \(p\) linearly increases from 0.2 to 0.8 between the 12,000th and 12,600th instances, then decreases back to 0.2 between the 18,000th and 19,200th instances, and finally increases again to 0.8 between the 24,000th and 25,200th instances. Outside these intervals, \(p\) remains constant, ensuring no drift occurs in the intervening segments. This design results in multiple incremental drifts across the data set.

These linear transitions facilitate incremental drifts, modifying the data distribution across two Gaussian components.

 \textbf{\underline{Primary Task 4-GM\_Grad}}: In the gradual drift task, \( p \) fluctuates between 0.2 and 0.8. The drift periods where \( p = 0.8 \) occur during the intervals (10000, 11000), (12001, 15000), and (18000, 21000). Conversely, in the intervening periods (11001,12000), (15001,18000), and (21001,24000), \( p \) reverts to 0.2. This oscillation creates a gradual drift pattern by alternating phases of drift and stability.

Complex drift tasks, conversely, involve a mixture of distributions like Gamma, Lognormal, and Weibull. These distributions exhibit substantial overlap in their probability density functions and possess similar statistical characteristics, leading to lower discriminability and making the detection of drifts more challenging. These tasks are further compounded by introducing multiple drifts of different natures within the data stream. Also, each complex drift task involves a data set of 30,000 instances, where every dimension conforms to the same distribution pattern, ensuring consistency across the multidimensional data space.

\textbf{\underline{Complex Task 1-GamLog\_Sud}}: Initially, data is generated from a Gamma distribution (\(\text{Gamma}(1.5, 20)\)) across 5 dimensions. At the 21,000th instance, there is a sudden shift to a 5-dimensional Log-normal distribution with parameters \(\mu = \log(30) - 0.5\) and \(\sigma = 0.5\). This transition represents a complex and sudden drift, with the overlap between the two distributions making the drift challenging to identify.
    
\textbf{\underline{Complex Task 2-LogGamWei\_Sud}}: The data stream, consisting of 20 dimensions, initially follows a Log-normal distribution (\(\text{Lognormal}(\log(30) - 0.5, 0.5)\)). At the 15,000th instance, there is a sudden drift to a Gamma distribution (\(\text{Gamma}(1.5, 20)\)), and at the 24,000th instance, it transitions to a Weibull distribution (\(\text{Weibull}(1.5, 20)\)). These successive drifts add higher complexity.

\textbf{\underline{Complex Task 3-GamGM\_SudGrad}}: This data stream consists of 20 dimensions, each following the same distribution. Initially, the data follows a Gamma distribution (\(\text{Gamma}(2, 10)\)). At the 11,000th instance, a sudden drift occurs, transitioning the data to a Gaussian mixture. Subsequently, the task experiences gradual drifts, where the weighting factor \( p \) alternates between 0.2 and 0.8 during specific intervals. This alternation leads to shifting dominance between two Gaussian distributions (\(\mathcal{N}(20, 10^2)\) and \(\mathcal{N}(20, 50^2)\)), creating a complex pattern of both sudden and gradual drifts.

\subsubsection{\textbf{Real-World Data Sets}}
\label{apdx:rdatasets}
Here we provide detailed descriptions of the real-world data sets employed to evaluate the performance of our detector: \textit{INSECTS} and \textit{EEG}. These data sets are instrumental in assessing how effectively the detector adapts to real-world concept drift scenarios.

\textbf{\underline{INSECTS}}: The INSECTS data sets consist of optical sensor readings obtained from monitoring mosquitoes. Concept drifts are induced by varying temperatures, which affect the insects' activity levels in alignment with their circadian rhythms. This data set offers a dynamic setting of concept drifts. We utilized three specific data sets from the collection, each representing one or more distinct types of drift: (i) \textbf{INSECTS\_Sud:} This subset exhibits five sudden drifts, initiated at a temperature of 30°C, with a sudden shift to 20°C, and subsequently to approximately 35°C among other changes. The stream captures several rapid transitions in temperature, illustrating sudden concept drifts throughout. (ii) \textbf{INSECTS\_Grad:} Illustrating both gradual and incremental drifts, this data set simulates a scenario where temperatures slowly transition over time, presenting a nuanced evolution of environmental conditions. (iii) \textbf{INSECTS\_IncreRec:} Features a unique pattern of reoccurring incremental drifts, where temperature gradually increases or decreases in cycles. This data set is pivotal for studying the model's performance in scenarios where drift patterns repeat over time, challenging the detection mechanism with both incremental and reoccurring drift phenomena.

\textbf{\underline{EEG}}: The EEG data set encompasses multivariate time-series data from a single continuous EEG recording using the Emotiv EEG Neuroheadset over a span of 117 seconds. Eye states were captured through video recording concurrent with the EEG data collection and were subsequently annotated manually to indicate moments of eye closure ("1") and eye opening ("0"). This data set provides a sequential record of neurological activity, with values arranged in the order they were measured, presenting a unique challenge for detecting shifts in physiological states over time.

\subsection{Implementation Details}
\label{apdx:implementation_detail}
\subsubsection{\textbf{Implementation of compared algorithms}}
For \algname{}, we set $m = 30$ (or $50$ for a larger sub-window size such as in INSECTS\_IncreRec), $k = 10$, $\lambda = 1$, $\epsilon_{small} = 1$, $\epsilon_{big} = 10$, and $L=1$. The encoder function was implemented as a two-layer MLP with a ReLU activation function. For synthetic data sets, the two-layer encoder features 100 units in both hidden and output layers. For the INSECTS data set, reflecting its more complex drift types and higher dimensionality, the encoder dimensions are increased to 200 (hidden) and 150 (output). For the EEG data set, which involves smaller window sizes, the dimensions are adjusted to 150 (hidden) and 100 (output). In all data sets, the encoder function has been trained for a single epoch with a learning rate of 0.005 for every sliding window. For the implementation of baselines, we referred to \textit{alibi-detect}\,\cite{alibi-detect} and used sufficiently high permutations (200) for the statistical method and the default projection and the best epochs for the learning-based method.

\subsubsection{\textbf{Evaluation Scheme}} For a fair comparison, all algorithms followed the prequential evaluation scheme\,\cite{Gam13} and used the same sliding window with
the window size $W$ of 10\% of the total length of each data set and $N_{sub} = 10$. Each algorithm compares every new sub-window with the preceding one in a window to identify concept drift with a significance level of 0.05 if applicable.

\subsubsection{\textbf{Computing Platform}} All experiments were conducted on a Linux server equipped with an Intel Xeon CPU @ 2.20GHz, 12GB RAM, and 226GB of storage where Ubuntu 22.04 LTS, Python 3.10, and PyTorch 2.1.0+cu122 were installed. An NVIDIA Tesla V100-SXM2-16GB GPU was used for the deep learning-based algorithms.

\subsection{Additional Performance Analysis Results}
\label{apdx:additional_results} 
For subtle and slow drift types, we have considered adaptively adjusting the training process to achieve better detection outcomes. In this regard, an incremental drift was introduced by linearly transitioning the weighting factor \(p\) from 0.0 to 1.0 between the 15,000th and 24,000th instances. Specifically, for this incremental drift scenario, a specialized training strategy was implemented that alternates between exclusively using positive sample pairs and a mix of both positive and negative sample pairs, aimed at better adapting to and learning the nuanced shifts present. Keeping other parameters constant, \algname{} achieved a precision of 0.80, significantly outperforming other baseline algorithms with a maximum precision of 0.55. Similar to the visualization discussed in Section \ref{sec:overall_performance}, Figure \ref{fig:spcial_hm} illustrates the progression of MCD in this scenario.

In addition, Figures \ref{fig:heatmap_insects_grad} and \ref{fig:heatmap_insects_incre} show the heatmaps of MCD for the other two types of INSECTS data sets. While it is challenging to accurately identify the exact starting points for gradual and incremental drifts in real data streams, the heatmap shows higher MCD values around the true drift indicators (in the yellow boxes). Figure \ref{fig:ablation_study_plus} shows the ablation study results with F1 and MCC. Table \ref{tab:performance_metrics} compares the performance of \algname{} with ADWIN\,\cite{bifet2007learning} adopted for an unsupervised setting with unidimensional data stream.

\begin{figure}[!t]
\centering
\begin{subfigure}{\columnwidth}
    \centering
    \includegraphics[width=\columnwidth]{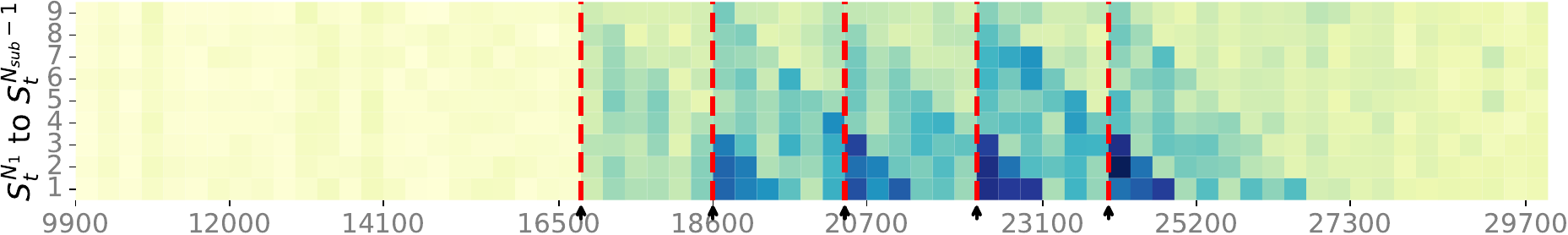}
    \vspace{-0.6cm}
    \caption{Heatmap of MCD for subtle and slow drifts.}
    \label{fig:spcial_hm}
\end{subfigure}
\hfill
\begin{subfigure}{\columnwidth}
  \centering
  \includegraphics[width=\linewidth]{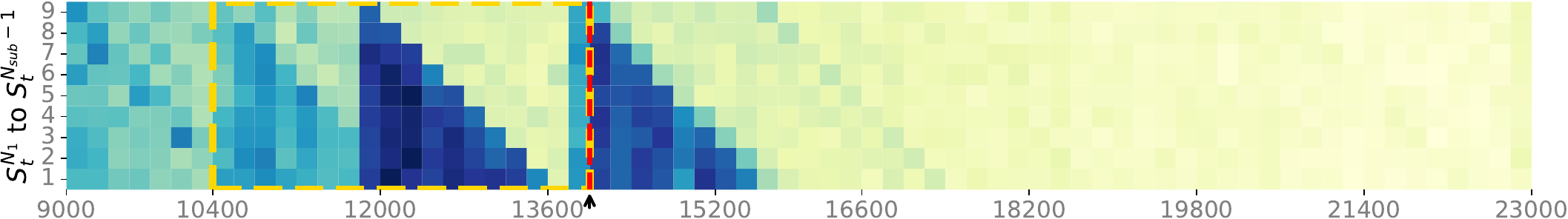}
  \vspace{-0.5cm}
  \caption{Heatmap for INSECTS\_Grad with drift indicators.}
  \label{fig:heatmap_insects_grad}
\end{subfigure}
\hfill
\begin{subfigure}{\columnwidth}
  \centering
  \includegraphics[width=\linewidth]{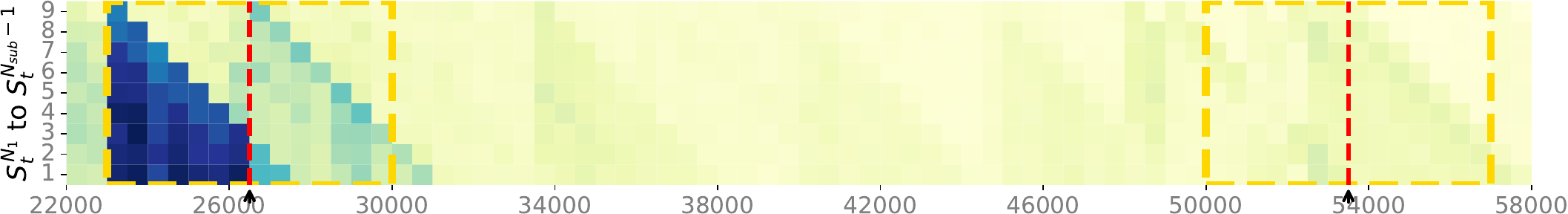}
  \vspace{-0.5cm}
  \caption{Heatmap for INSECTS\_IncreRec with drift indicators.}
  \label{fig:heatmap_insects_incre}
\end{subfigure}
  \vspace{-0.6cm}
\caption{Heatmaps of MCD for other data sets.}
\label{fig:heatmap_others}
\vspace{-0.2cm}
\end{figure}

\begin{figure}[!t]
    \centering
    \includegraphics[width=\columnwidth]{Figures/ablation_legend.pdf}
    \begin{subfigure}{0.9\columnwidth}
        \includegraphics[width=0.9\columnwidth]{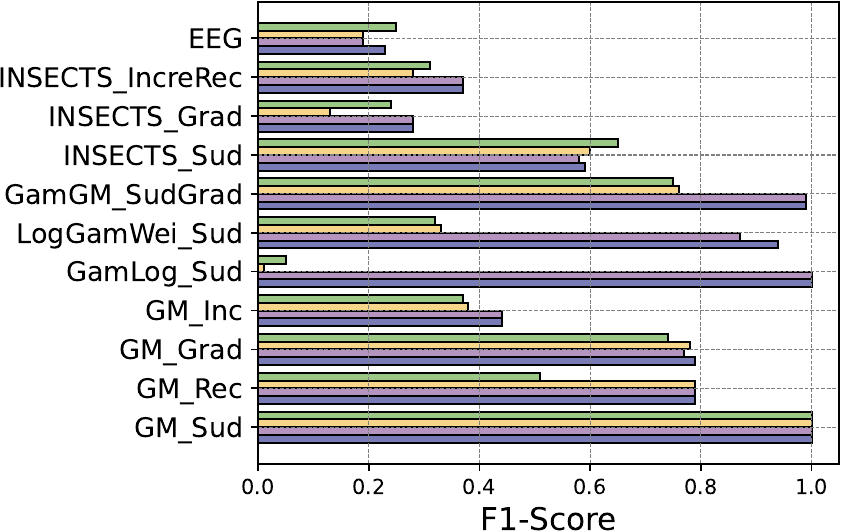}
        \label{fig:ablation_f1}
    \end{subfigure}
    \begin{subfigure}{0.9\columnwidth}
        \includegraphics[width=0.9\columnwidth]{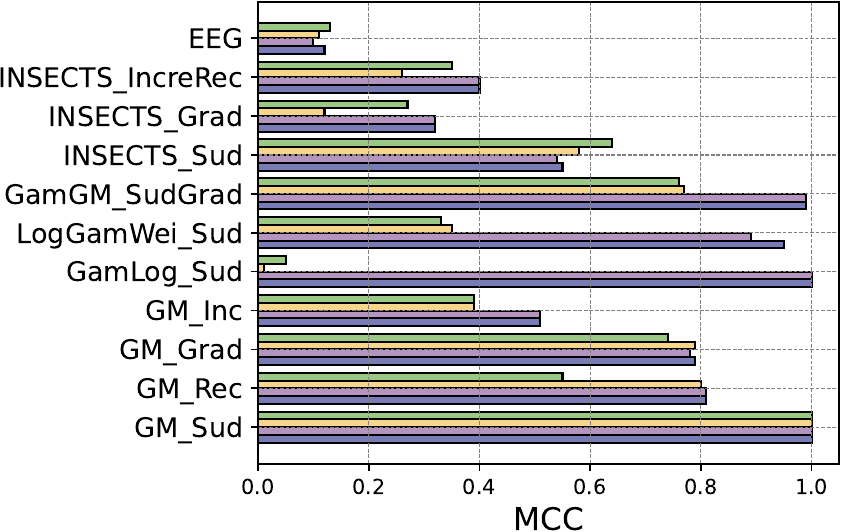}
        \label{fig:ablation_mcc}
    \end{subfigure}
    \vspace{-0.3cm}
    \caption{Ablation study results of F1 and MCC.}
    \label{fig:ablation_study_plus}
\vspace{-0.2cm}
\end{figure}





\begin{table}[t]
\small
\centering
\caption{Performance with unidimensional streams.}
\vspace{-0.3cm}
\label{tab:performance_metrics}
\begin{tabular}{@{}lcccccc@{}}
\toprule
 & \multicolumn{3}{c}{\textbf{\algname{}}} & \multicolumn{3}{c}{\textbf{ADWIN}} \\
 Data set & \textit{Pre.} & \textit{F1} & \textit{MCC}& \textit{Pre.} & \textit{F1} & \textit{MCC} \\
\midrule
INSECTS\_Sud & 0.585 & 0.503 & 0.472 & 0.667 & 0.244 & 0.332 \\
\bottomrule
\end{tabular}
\end{table}

\smallskip
\subsection{Details of Qualitative Analysis}
\label{apx:qualitative}
For the seven distributions, Dist1 is a normal distribution $\mathcal{N}(20, 10^2)$, generating samples in a 5-dimensional space. Dist2 is a normal distribution with increased variance $\mathcal{N}(20, 50^2)$. Dist3 is a mixture of two normal distributions, giving each sample a 50\% probability of originating from either $\mathcal{N}(20, 10^2)$ or $\mathcal{N}(20, 50^2)$. Dist4 is a uniform distribution spanning from 0 to 40, $\mathcal{U}(0, 40)$. Dist5 follows a gamma distribution with shape and scale parameters set to 2 and 10, respectively, $\text{Gamma}(2, 10)$. Dist6 utilizes a Weibull distribution with shape and scale parameters of 2 and 20, $\text{Weibull}(2, 20)$. Lastly, Dist7 is a log-normal distribution chosen to approximate a mean close to 20 by setting a standard deviation of 0.5 and a location parameter $\mu$ to $\log(20) - \frac{0.5^2}{2}$, $\text{Lognormal}(\mu, 0.5^2)$.


\end{document}